\documentclass{article}

\PassOptionsToPackage{sort&compress,numbers}{natbib}
\usepackage[letterpaper,left=1in,right=1in,top=1in,bottom=1in]{geometry}
\usepackage[utf8]{inputenc}
\usepackage[T1]{fontenc}
\usepackage[dvipsnames,table]{xcolor}
\usepackage{graphicx}
\usepackage{url}
\usepackage{booktabs}
\usepackage{amsfonts}
\usepackage{amssymb}
\usepackage{amsmath}
\usepackage{mathtools}
\usepackage{amsthm}
\usepackage{bm}
\usepackage{nicefrac}
\usepackage{microtype}
\usepackage{xspace}
\usepackage{multirow}
\usepackage{adjustbox}
\usepackage{array}
\usepackage{wrapfig}
\usepackage{needspace}
\usepackage{algorithm}
\usepackage{algorithmic}
\usepackage[shortlabels,inline]{enumitem}
\usepackage{natbib}
\usepackage[pagebackref,breaklinks,colorlinks,citecolor=blue,linkcolor=blue,urlcolor=blue]{hyperref}
\usepackage[capitalize,noabbrev]{cleveref}

\renewcommand{\algorithmicrequire}{\textbf{Input:}}
\renewcommand{\algorithmicensure}{\textbf{Output:}}

\usepackage{symbols}

\usepackage{amsmath,amsfonts,bm}

\def\1{\bm{1}}

\def\sp{space}

\def\ve{{\bm{e}}}

\def\vg{{\bm{g}}}

\def\vu{{\bm{u}}}
\def\vv{{\bm{v}}}
\def\vw{{\bm{w}}}

\def\mG{{\bm{G}}}

\def\mI{{\bm{I}}}

\def\mM{{\bm{M}}}

\def\mO{{\bm{O}}}
\def\mP{{\bm{P}}}

\def\mU{{\bm{U}}}
\def\mV{{\bm{V}}}
\def\mW{{\bm{W}}}

\DeclareMathAlphabet{\mathsfit}{\encodingdefault}{\sfdefault}{m}{sl}
\SetMathAlphabet{\mathsfit}{bold}{\encodingdefault}{\sfdefault}{bx}{n}

\newcommand{\R}{\mathbb{R}}

\DeclareMathOperator*{\argmin}{arg\,min}

\DeclareMathOperator{\sign}{sign}
\DeclareMathOperator{\msign}{msign}
\DeclareMathOperator{\diag}{diag}

\newcommand{\pa}[1]{\left( #1 \right)}
\newcommand{\bracks}[1]{\left[ #1 \right]}
\newcommand{\set}[1]{\left\{ #1 \right\}}
\newcommand{\inner}[2]{\left\langle #1,\, #2 \right\rangle}
\newcommand{\abs}[1]{\left\lvert #1 \right\rvert}

\usepackage{thmtools}
\usepackage{thm-restate}
\theoremstyle{plain}
\newtheorem{theorem}{Theorem}[section]
\newtheorem{proposition}[theorem]{Proposition}
\newtheorem{lemma}[theorem]{Lemma}

\theoremstyle{definition}

\newtheorem{assumption}[theorem]{Assumption}
\theoremstyle{remark}

\makeatletter
\DeclareRobustCommand{\pdot}{\mathbin{\mathpalette\pdot@\relax}}
\newcommand{\pdot@}[2]{\ooalign{$\m@th#1\circ$\cr
    \hidewidth$\m@th#1\cdot$\hidewidth\cr
  }}
\makeatother

\definecolor{darkred}{rgb}{0.7,0.1,0.1}
\definecolor{darkgreen}{rgb}{0.1,0.7,0.1}
\definecolor{dblue}{rgb}{0.2,0.2,0.8}
\definecolor{maroon}{rgb}{0.76,.13,.28}
\definecolor{burntorange}{rgb}{0.81,.33,0}
\definecolor{tealblue}{rgb}{0.212,0.459, 0.533}
\definecolor{mint}{rgb}{0.24, 0.71, 0.54}
\definecolor{mypink}{rgb}{0.93359375, 0.62109375, 0.83984375}

\definecolor{pp}{rgb}{0.43921569, 0.18823529, 0.62745098}
\definecolor{rr}{rgb}{0.5254902 , 0.00784314, 0.12941176}
\definecolor{bb}{rgb}{0.09019608, 0.23529412, 0.37647059}
\definecolor{yy}{rgb}{0.49803922, 0.3372549 , 0.0}
\definecolor{gg}{rgb}{0.02352941, 0.3372549 , 0.17647059}

\ifdefined\ShowNotes
  \newcommand{\colornote}[3]{{\color{#1}\bf{#2: #3}\normalfont}}
\else
  \newcommand{\colornote}[3]{}
\fi

\newcommand{\eat}[1]{}

\usepackage{thmtools}
\usepackage{thm-restate}
\usepackage{csquotes}

\definecolor{fail}{RGB}{236,120,115}
\definecolor{succ}{RGB}{119,205,255}
\newcommand{\myparagraph}[1]{\vspace*{0pt}{\bfseries\noindent #1}}

\title{\rule{\linewidth}{1.5pt}\\ \vspace{3pt}\textbf{Spectral Saliency for Machine Unlearning}\vspace{3pt}\\\rule[8pt]{\linewidth}{1pt}}
\author{
\normalsize
\textbf{Cedar Site~Bai},
\textbf{Amber Yijia~Zheng},
\textbf{Raymond A.~Yeh},
\textbf{Brian~Bullins}\\
Department of Computer Science\\
Purdue University\\
\small\texttt{\{bai123, zheng709, rayyeh, bbullins\}@purdue.edu}
}
\date{}
\let\cite\citep

\begin{document}

\maketitle

\begin{abstract}
   Machine unlearning (MU) aims to remove the influence of specific training data while preserving model utility. As the name suggests, MU can be viewed as the inverse of learning, using gradient-based updates to reduce the influence of a forget-set by counteracting the previously learned behavior. Recently, Muon, a gradient descent variant, has been introduced. Muon applies spectral magnitude normalization to encourage exploration of rare directions and demonstrates promising performance. Inspired by Muon, we adopt the spectral view for unlearning and propose Spectral Saliency Unlearning (SSU). SSU thresholds weak singular components and updates only those directions supported by a confident unlearning signal. We further provide theoretical justification for this thresholding approach from the perspective of the forgetting-retention trade-off. Experiments across image classifiers, diffusion models, and LLMs demonstrate SSU's effectiveness.
\end{abstract}

\section{Introduction}

Machine unlearning (MU) studies how to remove the influence of specific training data from a trained model without resorting to full retraining~\citep{cao2015towards}. There is an increasing need for MU to satisfy data-governance requirements, such as the General Data
Protection Regulation (GDPR)~\citep{EuropeanParliament2016a}. Recent methods focus on approximate unlearning, which aims to efficiently reduce the influence of the forget set through a few gradient-based updates~\citep{fan2024salun, thudi2022unrolling, izzo2021approximate}. These methods typically optimize a forgetting objective that increases the model’s loss on the forget set or pushes a model's predictions toward misclassifications.

In this work, we conceptualize MU as the partial ``reversal'' of the optimization process.
In the deep learning era, learning can be viewed as using stochastic gradient descent (SGD) to incrementally accumulate the influence of training examples into the model parameters. Correspondingly, approximate unlearning can be viewed as undoing the effect of gradient descent on the designated forget set. UnrollingSGD~\citep{thudi2022unrolling} provides one concrete instantiation, analyzing unlearning through the lens of the SGD training dynamics. Motivated by the central role of SGD in deep learning, it is natural to ask how variants of SGD that alter the geometry and scaling of updates translate into the design of unlearning updates. 

A canonical variant is SignSGD~\cite{bernstein2018signsgd}, which replaces the stochastic gradient $\vg$ with its coordinate-wise sign,~\ie, $\sign\pa{\vg}$, thereby equalizing update magnitudes across coordinates and emphasizing directional information over scale. For matrix-structured parameters~\citep{he2016deep, vaswani2017attention}, recent optimizers such as Muon~\citep{jordan2024muon} extend this magnitude normalization beyond coordinates by applying the matrix sign in the spectral domain of layer gradients. This design is motivated in part by the idea that exploring ``rare'', weakly expressed directions can be beneficial for learning~\citep{jordan2024muon}. The approach is based on analyses of deep learning dynamics showing that training progresses along singular directions in order of singular value magnitude. Smaller singular values associated with weaker and more specific structure are learned later~\citep{kleinman2024critical,rahaman2019spectral,lampinen2018an}. 

From this perspective, we view the recent unlearning method SalUn~\citep{fan2024salun} as an analogue of SignSGD. Instead of uniformizing coordinate magnitudes, SalUn thresholds low-magnitude coordinates of the unlearning gradient. As SignSGD and Muon are both SGD-variants, where Muon promotes rare directions by equalizing magnitudes in the gradient spectrum, one might ask whether \textit{unlearning could similarly benefit} from suppressing weak spectral components. We hypothesize that these weak components may encode interactions where forgetting and retention are entangled. This intuition is formalized through the concept of forgetting-retention conflict in~\secref{sec:ssu_theory}.

We propose Spectral Saliency Unlearning (SSU), a singular value thresholding method that suppresses weak spectral components of the unlearning gradient and can be plugged into general gradient-based unlearning methods. We justify SSU from the perspective of the forgetting-retention trade-off, characterizing how weak directions relate to forget-retain interference. We then explain why suppressing them can mitigate utility degradation, \ie, the loss of model performance on data not in the forget set. We then extend the same perspective to coordinate-wise masking and offer a theoretical explanation for the efficacy of SalUn-style unlearning, which has largely been used as a heuristic. 

Empirically, we evaluate SSU across diverse applications, including the unlearning of image classifiers, diffusion models for image generation~\cite{fan2024salun}, and large language models~\cite {yuan2025closer}. Results demonstrate consistent improvements when SSU is used with existing MU objectives. 
For image classifiers, SSU reduces the average gap by 9\% and 18\% at forgetting ratios of 10\% and 50\%, respectively, for CIFAR-10, and achieves a 27.7\% reduction on Tiny-ImageNet, with improvements across different random seeds. For diffusion models, SSU achieves perfect unlearning efficacy with DDPM while improving generation quality on retained classes by 23.6\%. For large language models, it consistently improves the utility-forgetting trade-off, yielding an average gain of 0.0125 on the TOFU benchmark.
{\bf\noindent Our main contributions are:}
\vspace{-3pt}
\begin{itemize}[topsep=0pt, leftmargin=12pt]
    \setlength{\itemsep}{0.0pt}
    \setlength{\parskip}{2.5pt}
    \item Inspired by recent advances in optimization, we introduce spectral saliency for matrix gradients and propose SSU, a novel singular value thresholding approach for gradient-based unlearning.
    \item We provide a justification for saliency-based thresholding from the forgetting-retention trade-off perspective, and extend the analysis to explain the efficacy of SalUn-style coordinate masking.
    \item Extensive experiments demonstrate the effectiveness of SSU across a range of models, including image classifiers, diffusion models, and LLM unlearning. \end{itemize}

\section{Background}

\myparagraph{Unlearning formulation.}
We consider a general gradient-based approximate unlearning framework~\cite{thudi2022unrolling, neel2021descent,gandikota2023erasing,kumari2023ablating,gandikota2024erasing}. A pre-trained model is updated to reduce the influence of a given forget set $\cD_f$ from the training data $\cD$ while preserving utility on the complementary retain set $\cD_r \coloneqq \cD \setminus \cD_f$. Let $\mW$ denote model parameters. We define an unlearning objective
\begin{align} \label{eq:unlearning}
    \cL_u(\mW)=\cL_f(\mW;\cD_f)+\cL_r(\mW;\cD_r)
\end{align}
composed of a forgetting loss $\cL_f(\mW;\cD_f)$ and a retaining loss $\cL_r(\mW;\cD_r)$, which are optimized jointly via gradient-based updates to minimize~\equref{eq:unlearning}.
We further denote the gradient update $\nabla \cL_u\pa{\mW}$ at each step as $\mG_u = \mG_f + \mG_r$, 
where $\mG_f=\nabla \cL_f(\mW;\cD_f)$ and $\mG_r=\nabla \cL_r(\mW;\cD_r)$.  

\myparagraph{Muon and matrix sign.} Muon~\cite{jordan2024muon} is a recently proposed optimizer for matrix-parameterized functions. It orthogonalizes the layer gradient by the projection 
\bea
\mG_o = \argmin_{\mO}\set{\norm{\mG-\mO}_F: \mO^\top \mO \ \text{or} \ \mO \mO^\top = \mI}.
\eea
This operation is equivalent to applying the matrix sign in the gradient's spectral domain \citep{chen2025muon},~\ie, for $\mG = \mU\diag\pa{\bm{\sigma}}\mV$, then we have $\msign\pa{\mG} \coloneqq \mU\diag\pa{\sign\pa{\bm{\sigma}}}\mV$. We therefore view Muon as a matrix counterpart of SignSGD.

\section{Method} \label{sec:ssu_method}
\begin{wrapfigure}[14]{R}{0.54\textwidth}
\begin{minipage}{\linewidth}
\setlength{\intextsep}{0pt}
\begin{algorithm}[H]
\small
   \caption{Spectral Saliency Unlearning (SSU)}  
   \label{alg:ssu}
\begin{algorithmic}[1]
   \INPUT 
   Pre-trained parameters $\mW_0$; unlearning objective $\cL_u(\mW)=\cL_f(\mW)+\cL_r(\mW)$, learning rate $\eta$, keep ratio $\gamma\in(0,1]$
   \FOR{$t = 0, 1,\dots,T-1$}  
       \STATE $\mG_{u,t} \leftarrow \nabla_{\mW}\cL_u(\mW_t)$. 
       \FOR{each matrix-structured gradient $\mG_{u,t}^{\pa{\ell}}$}
       \STATE Compute SVD: $\mG_{u,t}^{\pa{\ell}} = \mU \diag(\bm{\sigma}) \mV^\top$
        \STATE Set $k \leftarrow \lfloor \gamma m \rfloor$ for $\bm{\sigma} \in \R^m$.
        \STATE $\widetilde{\mG}_{u,t}^{(\ell)} \leftarrow \mU_{[:,1:k]}\,\diag(\bm{\sigma}_{[1:k]})\,\mV_{[:,1:k]}^\top$
    \ENDFOR
    \STATE $\mW_{t+1} \leftarrow \mW_t - \eta\widetilde{\mG}_{u,t}$
   \ENDFOR
   \OUTPUT $\mW_T$
\end{algorithmic}
\end{algorithm}
\end{minipage}
\vspace{-\intextsep}
\vspace{-0.58cm}
\end{wrapfigure}
For a gradient-based unlearning algorithm that minimizes the objective given in \equref{eq:unlearning}, our proposed SSU replaces the base unlearning gradient $\mG_u$ with a singular value thresholded counterpart that suppresses weak spectral components and retains only the dominant singular directions for the update. The full procedure is summarized in~\algref{alg:ssu}. In the following, we elaborate on this procedure by first defining spectral saliency and then describing the singular value thresholding step.

\Needspace{4\baselineskip}
\subsection{Spectral Saliency} 
We define spectral saliency by measuring the magnitude of the unlearning update $\mG_u$ along the singular directions of the layer gradient. For a matrix-structured parameter $\mW$, let the unlearning gradient's SVD be denoted as 
\begin{align}
    \begin{aligned}
    \mG_u &= \mU\,\diag(\bm{\sigma})\,\mV^\top \\
    \text{with}\quad \bm{\sigma} &= \bracks{\sigma_1,  \cdots, \sigma_m}^\top,
    \end{aligned}
\end{align}
where $\sigma_1\ge \cdots \ge \sigma_m\ge 0$, so that each rank-one component $\sigma_i\,\vu_i\vv_i^\top$ represents the update along the singular direction pair $(\vu_i,\vv_i)$. Note, we omit layer and block indices for readability.

The spectral saliency of direction $i$ is defined as its singular magnitude $\sigma_i$, which quantifies how strongly the unlearning objective drives updates along this direction. We view larger saliency as corresponding to a stronger and more reliable unlearning signal, whereas small singular magnitudes indicate weakly supported directions that are less confident for unlearning and may be susceptible to interference between forgetting and retention, an intuition we later formally justify in~\secref{sec:ssu_theory}.

\subsection{Singular Value Thresholding}
To suppress these weak directions, we apply \emph{singular value thresholding} \citep{cai2010singular} to $\mG_u$.  Given a threshold $\tau\geq 0$, we define the thresholded gradient
\begin{align}
    \widetilde{\mG}_u \coloneqq \mU \diag\pa{\bm{\sigma} \odot \mathbb{I}_{\bracks{\bm{\sigma} \geq \tau}}} \mV^\top,
    \;\;\text{with}\;\;
     \mathbb{I}_{\bracks{\bm{\sigma} \geq \tau}}(i) = 
    \begin{cases}
                1 & \text{if } \sigma_i \geq \tau \\
                0 & \text{otherwise}\\
    \end{cases}.
\end{align}
$\mathbb{I}_{\bracks{\bm{\sigma} \geq \tau}}$ is a vector-valued indicator whose $i^{th}$ coordinate is $1$
when $\sigma_i \geq \tau$ and $0$ otherwise. Equivalently, we may retain the leading components and discard the tail as the singular values are sorted.

Different from coordinate-wise saliency methods, \eg, SalUn~\cite{fan2024salun}, which constructs a coordinate selection mask using dataset-specific computations (\eg, based on the forget set) either offline or online, SVD provides an intrinsic ordering of singular values. The most salient directions correspond naturally to the leading singular components. We use a fixed \emph{keep ratio} $\gamma\in(0,1]$ and set $k=\lfloor \gamma m\rfloor$, retaining the top $\gamma$ fraction of singular components. We then reconstruct
\begin{align}
    \widetilde{\mG}_u \coloneqq
    \sum_{i=1}^{k} \sigma_i\,\vu_i\vv_i^\top =\mU_{[:,1:k]}\,\diag(\bm{\sigma}_{[1:k]})\,\mV_{[:,1:k]}^\top
\end{align}
and proceed with standard gradient-based updates using this thresholded gradient.
For vector- and scalar-structured parameters (\eg, biases), we resort to coordinate-wise thresholding as in SalUn, mirroring Muon’s practice of handling these parameters with vector-based optimizers rather than a spectral update.

\section{Theoretical Justification} \label{sec:ssu_theory}
Our analysis adopts a directional view of the unlearning update, characterizing when and why suppressing weakly supported directions can improve unlearning from the perspective of the forgetting-retention trade-off. We first develop this justification for spectral thresholding on matrix-structured parameters, then derive an analogous interpretation for coordinate-wise masking.

\subsection{Why Spectral Thresholding Helps}
Motivated by the Muon perspective of operating in the gradient’s spectral domain, and by analyses linking smaller singular values to weaker, more specific structure~\citep{kleinman2024critical, rahaman2019spectral, lampinen2018an}, we take the reverse stance for unlearning. Weak spectral components tend to be the least reliable directions to update for unlearning, as they may reflect forget-retain entanglement and thereby worsen the forgetting-retention trade-off. Accordingly, we use the singular basis of the unlearning gradient to explain why suppressing these components via spectral thresholding can improve unlearning.
We start by examining how the forgetting gradient and the retaining gradient contribute to and interact within the unlearning gradient. 
We consider the compact SVD of the unlearning gradient, retaining only non-zero singular values:
\vspace{-0.1cm}
\bea
\mG_u = \mU_e\diag(\bm{\sigma}_e)\,\mV_e^\top
= \sum_{i=1}^{r_e} \sigma_i\,\vu_i\vv_i^\top, \sigma_1\ge\cdots\ge\sigma_{r_e}>0
\vspace{-0.1cm}
\eea
where $r_e$ denotes the effective rank: the exact rank in theory and the number of significantly nonzero singular values in practice. We refer to $\mathcal{E}=\mathrm{span}(\mV_e)$ as the effective subspace of update directions induced by $\mG_u$. In this subspace, each singular direction $\vv_i$ defines a canonical one-dimensional update mode with a magnitude $\sigma_i$.

To understand how the forgetting and retention gradients interact within the effective subspace, we introduce the notion of \emph{forget-retain alignment}, which measures the directional alignment between the projected forgetting and retention gradients. 
\begin{restatable}{definition}{defalignment} {(Forget-Retain Alignment)} \label{def:alignment}
    We define the alignment between the forgetting and retention directions of $\vv_i$ as 
    \vspace{-0.1cm}
      \bea
      a\pa{\vv_i} = \frac{\inner{\mG_f\vv_i}{\mG_r\vv_i}}{\norm{\mG_f\vv_i}_2\norm{\mG_r\vv_i}_2}.
      \vspace{-0.1cm}
      \eea
\end{restatable}
Using this notion, Prop.~\ref{prop:conflict} shows that when the forgetting and retaining gradients exhibit strong negative alignment in the effective subspace, this indicates that there is significant conflict between the forgetting and retention objectives, which therefore cannot be simultaneously improved along the same update direction, yielding an unavoidable trade-off.

\begin{proposition} \label{prop:conflict}
(Informal) Assume $\cL_f$ and $\cL_r$ are smooth. For unlearning gradient $\mG_u = \eta \sum_{i=1}^{r_e} \sigma_i \vu_i\vv_i^\top$, consider a single-direction update along the $i^{th}$ singular component $\vv_i$: $\mW' = \mW - \eta \sigma_i \vu_i\vv_i^\top$. There exist constants $a_0 < 0$ and $\eta_0>0$ such that if the alignment $a\pa{\vv_i} < a_0$, and the step size $\eta < \eta_0$, then the forgetting progress $\Delta \cL_f = \cL_f\pa{\mW'} - \cL_f\pa{\mW}$ and the retaining progress $\Delta \cL_r = \cL_r\pa{\mW'} - \cL_r\pa{\mW}$ satisfy $\Delta \cL_f \Delta \cL_r < 0$. That is, along any singular direction $\vv_i$ where the retain and forget gradients are in significant conflict, any update that improves one objective must necessarily worsen the other.
\end{proposition}
We refer the readers to Appx.~\ref{sec:appendix_ssu_1}  for the formal statement, assumptions, and proof of this proposition. 
Next, we show how this conflict is reflected in the spectrum of the unlearning gradient.
\begin{proposition} \label{prop:small_singular}
{\bf (a)} If the direction carries a nontrivial forget/retain signal, \ie, $\norm{\mG_f \vv_i}+\norm{\mG_r \vv_i} \geq \xi$ for $\xi > 0$ and its singular value is small relative to this signal, \ie, $\sigma_i \leq \rho \xi$ for $\rho \in \pa{0, \frac{1}{\sqrt{2}}}$, then its forget-retain alignment satisfies $a\pa{\vv_i} \leq 2\rho^2-1 < 0$.
 \\
{\bf (b)} (Informal) Assume the projections of $\mG_f$ and $\mG_r$ onto the effective subspace of $\mG_u$ have bounded spectral disparity. For two directions $\vv_i,\vv_j$ with alignment scores $a\pa{\vv_i} \leq -\delta_i$ and $a\pa{\vv_j} \geq \delta_j$ where $\delta_i, \delta_j \in (0, 1]$, if the separation $\pa{\delta_i + \delta_j}$ exceeds a gap threshold, then $\sigma_i < \sigma_j$.
\end{proposition}
We refer the readers to Appx.~\ref{sec:appendix_ssu_2} for the assumption of spectral disparity, the formal statement, and proof. 
Prop.~\ref{prop:small_singular} $(a)$ shows that as long as there is a nontrivial forgetting or retention signal, small singular values correspond to directions in which the forgetting and retaining gradients are negatively aligned. If a direction carries neither forgetting nor retention signal, \ie, $\norm{\mG_f \vv_i} = \norm{\mG_r \vv_i} = 0$, then $\vv_i$ lies in the null space of $\mG_u$, which is not affected by singular value thresholding. Prop.~\ref{prop:small_singular} $(b)$ complementarily shows that, between sufficiently positively aligned directions and negatively aligned ones, the former attain larger singular values. Together with Prop.~\ref{prop:conflict}, this analysis shows that by suppressing small singular components of $\mG_u$, SSU preferentially removes directions that are more likely to exhibit strong forget-retain conflict, thereby mitigating unnecessary retention degradation while preserving the dominant unlearning signal. 

\subsection{Justification for SalUn-Style Masking}
We further extend the same directional analysis to SalUn-style coordinate thresholding, which operates on the unlearning gradient in the standard coordinate basis. Each coordinate $i\in[d]$ defines a one-dimensional update direction, and the unlearning gradient decomposes as $g_{u,i}=g_{f,i}+g_{r,i}$. A key regime of interest is when a coordinate carries a nontrivial forgetting or retention signal (otherwise updates along that coordinate are negligible), yet the magnitude $|g_{u,i}|$ is small, suggesting cancellation between forgetting and retention effects along that coordinate. This explains why SalUn’s heuristic of suppressing small-magnitude coordinates can be beneficial: updating along such weakly supported coordinates is likely to couple forgetting progress with retention degradation, as characterized by the following proposition whose proof can be found in Appx.~\ref{sec:appendix_salun}.
\begin{restatable}{proposition}{propsalun} \label{prop:salun}
    Assume $\cL_f$ and $\cL_r$ are $\beta_f$- and $\beta_r$-smooth. Consider any coordinate $i\in[d]$ such that the forgetting/retention signal is nontrivial, \ie, $\forall i \in [d]$, $\abs{g_{f,i}} + \abs{g_{r,i}} \geq \xi$ where $\xi > 0$, and the combined unlearning gradient is small, \ie, $\abs{g_{u,i}} \leq \rho \xi$ for some $\rho\in(0,1]$. For the coordinate-wise update $\vw' = \vw - \eta g_{u,i}$ with step size $\eta < \min\pa{\frac{2\abs{g_{f,i}}}{\beta_f\abs{g_{u, i}}}, \frac{2\abs{g_{r,i}}}{\beta_r\abs{g_{u, i}}}}$, the induced progress $\Delta\cL_f = \cL_f\pa{\vw'} - \cL_f\pa{\vw}$ and $\Delta\cL_r = \cL_r\pa{\vw'} - \cL_r\pa{\vw}$ satisfy $\Delta\cL_f \Delta \cL_r < 0$.
\end{restatable}
That is to say, the unlearning update necessarily improves one objective while worsening the other, between forgetting and retention.
To our knowledge, existing saliency-masking methods are primarily motivated empirically. The proposition above offers a nontrivial theoretical justification from the forgetting-retention trade-off perspective.

{\bf Limitations and discussion.}
Our results justify SSU as a principled mechanism for mitigating forget-retain conflict by suppressing weak spectral components. At the same time, the theory is intended as a characterization of the forgetting-retention trade-off rather than a complete quantitative prediction of end-to-end unlearning performance. Extending the analysis to weaker assumptions and deriving tighter performance-level guarantees remain important directions for future work. 
A natural variant of SSU would be to threshold directions by forget-retain alignment. However, this requires separately computing $\mG_f$ and $\mG_r$, projecting both onto the singular directions of $\mG_u$, and then evaluating their directional inner products, introducing nontrivial computation overhead. Singular values, in contrast, are obtained directly from $\mG_u$ and, supported by Prop.~\ref{prop:small_singular}, serve as a principled proxy for forget-retain conflict.
While a more fine-grained, alignment-aware criterion may improve thresholding, we focus on the simple, broadly applicable choice of singular value thresholding.

\section{Experiments} \label{sec:exp}

We conduct experiments across three applications spanning over image classification (\secref{sec:image_class}), image generation (\secref{sec:image_gen}), and language modeling (\secref{sec:unlearn_llm}) following existing MU works.

\subsection{Random Subset Unlearning in Image Classification}\label{sec:image_class}

\myparagraph{Setup and evaluation.}
Following~\citet{fan2024salun}, we focus on random subset unlearning in image classification tasks using the CIFAR-10~\cite{krizhevsky2009learning} and Tiny-ImageNet datasets~\cite{le2015tiny}. 
We employ ResNet-18~\cite{he2016deep} as our architecture and compare our method against three simple baselines: fine-tuning (FT)~\cite{warnecke2021machine}, gradient ascent (GA)~\cite{thudi2022unrolling}, influence unlearning (IU)~\cite{izzo2021approximate}, and six competitive baselines: $\ell_1$-sparse~\cite{jia2023model}, SCRUB~\cite{kurmanji2023towards}, SSD~\cite{foster2024fast}, SEMU~\cite{sendera2025semu}, SFRON~\cite{huang2024unified}, and SalUn~\cite{fan2024salun}. 

Following SalUn's setup, we regard the retrained oracle as the gold standard of unlearning. To compare each method with this reference, we report four metrics: forgetting set accuracy (FA) to measure unlearning efficacy, remaining set accuracy (RA) and test set accuracy (TA) to assess preserved generalization, and the membership inference attack (MIA)~\cite{fan2024salun} success rate on the forgetting set as a privacy metric. We measure each metric's distance from its corresponding oracle value and report the average gap across the four metrics as the overall performance. Implementation details are provided in Appx.~\ref{sec:appendix_image_cls}.

\myparagraph{Results.} In~\tabref{tab:cifar10_classification}, we summarize the results for random subset unlearning on CIFAR-10 with ResNet-18,  where 10\% of the training data is forgotten. Most baseline methods face a fundamental trade-off: techniques such as FT, GA, IU, SCRUB, SSD, SEMU, and SFRON yield high FA, indicating insufficient forgetting, while methods like $\ell_1$-sparse that achieve lower FA suffer large drops in RA and TA, compromising model utility. 
\par
\Needspace{11\baselineskip}
\begin{wraptable}[9]{r}{0.5\linewidth}
\vspace{-\intextsep}
\centering
\caption{CIFAR-10, 10\% random data forgetting (Mean
$\pm$ std over five seeds).}
\label{tab:cifar10_classification}
\resizebox{\linewidth}{!}{
\begin{tabular}{l|cccccc}
\specialrule{.15em}{.05em}{.05em}
Methods & FA & RA & TA & MIA & Avg. Gap $\downarrow$ & Time (min) $\downarrow$ \\
\midrule
Retrain & 94.57$\pm$0.57 & 100.00$\pm$0.00 & 94.25$\pm$0.10 & 13.35$\pm$0.57 & 0.00 & 43.29 \\
\midrule
FT & 99.91$\pm$0.04 & 99.95$\pm$0.04 & 94.11$\pm$0.15 & 1.29$\pm$0.15 & 4.42$\pm$0.27 & 2.37 \\
GA & 100.00$\pm$0.00     & 100.00$\pm$0.00 & 94.66$\pm$0.06 & 0.49$\pm$0.10 & 4.68$\pm$0.27 & 0.13 \\
IU & 100.00$\pm$0.00 & 100.00$\pm$0.00 & 94.63$\pm$0.10 & 0.39$\pm$0.10 & 4.70$\pm$0.26 & 3.22 \\
$\ell_{1}$-sparse & 94.81$\pm$0.54 & 97.16$\pm$0.43 & 90.76$\pm$0.65 & 10.95$\pm$0.91 & 2.38$\pm$0.23 & 2.36 \\
SCRUB & 100.00$\pm$0.00 & 100.00$\pm$0.00 & 94.69$\pm$0.11 & 0.40$\pm$0.11 & 4.71$\pm$0.25 & 1.88 \\
SSD & 99.48$\pm$0.70 & 99.46$\pm$0.63 & 93.31$\pm$1.28 & 2.47$\pm$1.77 & 4.31$\pm$0.14 & 2.78 \\
SEMU & 100.00$\pm$0.00 & 100.00$\pm$0.00 & 94.61$\pm$0.06 & 1.00$\pm$0.32 & 4.54$\pm$0.32 & 1.05 \\
SFRON & 97.26$\pm$0.30 & 99.79$\pm$0.03 & 94.13$\pm$0.08 & 8.24$\pm$0.71 & 2.04$\pm$0.12 & 1.90 \\
SalUn & 97.03$\pm$0.55 & 99.28$\pm$0.26 & 93.16$\pm$0.45 & 13.78$\pm$0.86 & 1.21$\pm$0.25 & 2.61\\
\hline
\rowcolor{gray!20}
+ SSU & 97.06$\pm$0.56 & 99.38$\pm$0.21 & 93.26$\pm$0.38 & 13.38$\pm$0.67 & \bf 1.10$\pm$0.20 & 2.62 \\
\specialrule{.15em}{.05em}{.05em}
\end{tabular}
}

\end{wraptable}
Among the baselines, SalUn achieves the best overall performance with an average gap of 1.21. Our method improves upon SalUn across all five different seeds, achieving the lowest average gap of 1.10 with a paired t-test p-value of 0.034, while maintaining comparable performance across all metrics. Although SSU employs SVD, the actual computational overhead is negligible during unlearning due to the low-rankness in model weights.

\begin{table}[H]
\centering
\begin{minipage}[b]{0.49\linewidth}

\centering
\caption{CIFAR-10, 50\% random data forgetting (Mean $\pm$ std over five seeds).}
\label{tab:cifar10_classification_50pct}
\resizebox{\linewidth}{!}{
\begin{tabular}{l|ccccc}
\specialrule{.15em}{.05em}{.05em}
{Methods} & {FA} & {RA} & {TA} & {MIA} & {Avg.\ Gap $\downarrow$} \\
\midrule
{Retrain} & {91.07$\pm$0.24} & {100.00$\pm$0.00} & {90.70$\pm$0.20} & {20.77$\pm$0.60} & {0.00} \\  \midrule
{FT} & {99.91$\pm$0.04} & {99.97$\pm$0.02} & {94.02$\pm$0.33} & {1.54$\pm$0.18} & {7.85$\pm$0.32}  \\  {GA} & {99.93$\pm$0.09} & {99.94$\pm$0.08} & {94.31$\pm$0.52} & {0.64$\pm$0.20} & {8.19$\pm$0.12} \\  {IU} & {100.00$\pm$0.00} & {100.00$\pm$0.00} & {94.64$\pm$0.10} & {0.41$\pm$0.09} & {8.31$\pm$0.25} \\  {$\ell_{1}$-sparse} & {98.82$\pm$0.49} & {99.34$\pm$0.39} & {92.54$\pm$0.55} & {4.58$\pm$1.07} & {6.61$\pm$0.50} \\  {SCRUB} & {100.00$\pm$0.00} & {100.00$\pm$0.00} & {94.67$\pm$0.16} & {0.47$\pm$0.06} & {8.28$\pm$0.27} \\  {SSD} & {99.99$\pm$0.02} & {99.99$\pm$0.01} & {94.48$\pm$0.12} & {0.66$\pm$0.13} & {8.20$\pm$0.25}  \\  {SEMU} & {99.79$\pm$0.29} & {99.76$\pm$0.31} & {93.36$\pm$0.89} & {3.90$\pm$0.59} & {7.12$\pm$0.31}  \\  {SFRON} & {95.87$\pm$0.21} & {99.95$\pm$0.01} & {93.04$\pm$0.13} & {11.68$\pm$0.35} & {4.07$\pm$0.18}  \\  {SalUn} & {96.04$\pm$1.84} & {97.37$\pm$1.66} & {91.34$\pm$1.61} & {14.09$\pm$1.79} & {3.91$\pm$0.76}  \\  \hline
\rowcolor{gray!20}
{+ SSU} & {93.07$\pm$2.06} & {94.86$\pm$1.93} & {89.43$\pm$1.56} & {16.86$\pm$2.60} & {\textbf{3.21$\pm$0.68}} \\ \specialrule{.15em}{.05em}{.05em}
\end{tabular}
}

\par\vspace{0pt}
\end{minipage}\hfill
\begin{minipage}[b]{0.49\linewidth}

\centering
\caption{Tiny ImageNet, 10\% random data forgetting (Mean ± std over three seeds).}
\label{tab:tiny_imagenet_classification}
\resizebox{\linewidth}{!}{
\begin{tabular}{l|ccccc}
\specialrule{.15em}{.05em}{.05em}
{Methods} & {FA} & {RA} & {TA} & {MIA} & {Avg.\ Gap $\downarrow$} \\
\midrule
{Retrain} & {59.42$\pm$0.51} & {99.98$\pm$0.00} & {59.90$\pm$0.04} & {73.97$\pm$0.69} & {0.00}  \\  \midrule
{FT} & {99.96$\pm$0.01} & {99.98$\pm$0.00} & {61.08$\pm$0.59} & {1.87$\pm$0.34} & {28.46$\pm$0.38}  \\  {GA} & {99.98$\pm$0.01} & {99.98$\pm$0.01} & {60.79$\pm$0.54} & {1.39$\pm$0.45} & {28.51$\pm$0.44}  \\  {IU} & {99.98$\pm$0.01} & {99.98$\pm$0.00} & {61.39$\pm$0.57} & {1.12$\pm$0.26} & {28.73$\pm$0.41}  \\  {$\ell_{1}$-sparse} & {99.96$\pm$0.02} & {99.98$\pm$0.00} & {60.42$\pm$0.40} & {3.36$\pm$0.86} & {27.92$\pm$0.28}  \\  {SCRUB} & {99.97$\pm$0.02} & {99.98$\pm$0.00} & {61.47$\pm$0.49} & {1.54$\pm$0.21} & {28.64$\pm$0.38}  \\  {SSD} & {99.97$\pm$0.00} & {99.98$\pm$0.00} & {61.43$\pm$0.52} & {1.04$\pm$0.20} & {28.76$\pm$0.43}  \\  {SEMU} & {5.73$\pm$0.82} & {5.74$\pm$0.69} & {5.86$\pm$0.53} & {2.57$\pm$1.66} & {68.34$\pm$0.94}  \\  {SFRON} & {98.83$\pm$0.33} & {99.98$\pm$0.00} & {60.67$\pm$0.22} & {23.50$\pm$4.11} & {22.66$\pm$0.83}  \\  {SalUn} & {74.08$\pm$3.57} & {99.76$\pm$0.09} & {53.30$\pm$0.34} & {82.06$\pm$1.71} & {7.40$\pm$0.57}  \\  \hline
 \rowcolor{gray!20}
 {+ SSU} & {56.51$\pm$0.60} & {99.76$\pm$0.04} & {51.30$\pm$1.00} & {83.64$\pm$2.51} & {\bf 5.35$\pm$0.22}  \\  \specialrule{.15em}{.05em}{.05em}
\end{tabular}
}

\par\vspace{0pt}
\end{minipage}
\end{table}

For more challenging settings, ~\tabref{tab:cifar10_classification_50pct} shows the results for unlearning 50\% of CIFAR-10 with the same baselines and setup. We see that SSU remains the best method, reducing the average gap to 3.21$\pm$0.68, improving that of SalUn by 17.9\%. ~\tabref{tab:tiny_imagenet_classification} moves to Tiny-ImageNet with 200 classes, where 10\% of the training data is forgotten. The average gaps of all
non-saliency baselines stay above 22, while SSU reaches 5.35 $\pm$ 0.22, improving the result of SalUn by 27.7\%.

\subsection{Class-wise Unlearning in Image Generation}
\label{sec:image_gen}

\myparagraph{Setup and evaluation.}
Following the setup in~\cite{fan2024salun}, we evaluate our method on class-wise unlearning for image generation using DDPM~\cite{ddpm} on CIFAR-10 and Stable Diffusion V1.4 (SD)~\cite{sd} on Imagenette~\cite{imagenette}. We compare against four baselines: SA~\cite{heng2023selective}, ESD~\cite{gandikota2023erasing}, SFRON~\cite{huang2024unified}, and SalUn~\cite{fan2024salun}. For DDPM, SVD is applied only to the convolution layers of the denoising U-Net, while for Stable Diffusion, SVD is applied only to the cross-attention layers of the U-Net. 
We evaluate unlearning quality using the FID~\cite{Seitzer2020FID} score to measure generation quality on retained classes, and forgetting set accuracy (FA) of unlearned classes via a pre-trained classifier. Implementation details are provided in Appx.~\ref{sec:appendix_image_gen}.

\myparagraph{Results on DDPM.} \tabref{tab:ddpm_cifar} presents class-wise unlearning results for DDPM on  CIFAR-10 across
\begin{wraptable}[8]{r}{0.58\linewidth}
\vspace{-0.8cm}
\centering
\setlength{\tabcolsep}{6pt}
\renewcommand{\arraystretch}{1.15}
\caption{Class-wise unlearning of image generation on CIFAR-10 with DDPM.}
\resizebox{\linewidth}{!}{
\begin{tabular}{l|cc|cc|cc|cc|cc}
\specialrule{.15em}{.05em}{.05em}
\multirow{2}{*}{\textbf{Method}} &
\multicolumn{10}{c}{\textbf{CIFAR-10 Class-wise Unlearning}} \\
\cline{2-11}
& \multicolumn{2}{c|}{Automobile} & \multicolumn{2}{c|}{Cat} & \multicolumn{2}{c|}{Dog} &
  \multicolumn{2}{c|}{Horse} & \multicolumn{2}{c}{Truck}  \\
& \textbf{FA}$\downarrow$ & \textbf{FID}$\downarrow$ &
  \textbf{FA}$\downarrow$ & \textbf{FID}$\downarrow$ &
  \textbf{FA}$\downarrow$ & \textbf{FID}$\downarrow$ &
  \textbf{FA}$\downarrow$ & \textbf{FID}$\downarrow$ &
  \textbf{FA}$\downarrow$ & \textbf{FID}$\downarrow$ \\
\midrule
SA    & \textbf{0.00} & 23.56 & 14.20 & 21.34 &  8.60 & 21.19 & \textbf{0.00} & 21.13 & \textbf{0.00} & 29.04 \\

SFRON & \textbf{0.00} & 20.70 &  7.40 & 18.44 & 0.20 & 18.89 & \textbf{0.00} & 19.93 & \textbf{0.00} & 20.61 \\
SalUn  & 0.20          & 21.23 &  1.40 & 20.29 & \textbf{0.00} & 20.18 & 0.60 & 20.70 & 0.80 & 20.45 \\

\midrule
\rowcolor{gray!20}
+ SSU   & \bf 0.00          & \textbf{15.46} & \textbf{1.00} & \bf 16.01 & \textbf{0.00} & \textbf{15.19} & \textbf{0.00} & \textbf{15.60} & \textbf{0.00} & \textbf{16.32} \\
\specialrule{.15em}{.05em}{.05em}
\end{tabular}
}
\label{tab:ddpm_cifar}

\end{wraptable}
five classes. While all methods achieve near-perfect forgetting, they differ significantly in generation quality. 
Baseline methods yield FID scores ranging from 18.44 to 29.04, indicating moderate image quality on retained classes.
Our method achieves substantial improvements, with FID 
scores between 15.19 and 16.32 across all classes, 
which is a 25-30\% reduction compared to baselines. Notably, SSU maintains perfect forgetting on all classes while delivering the best generation quality, with the most significant improvements on Automobile, Cat, and Truck. These results show that our approach can effectively unlearn target data while preserving and even improving the generative capabilities. 

\begin{figure*}[t]
\centering
\small
\resizebox{\textwidth}{!}{\begin{tabular}{c*{10}{c}}
\toprule
 & \textbf{tench} & \textbf{springer} & \textbf{cassette} & \textbf{chain saw} & \textbf{church} & \textbf{French horn} & \textbf{garbage truck} & \textbf{gas pump} & \textbf{golf ball} & \textbf{parachute} \\
\midrule
\raisebox{\dimexpr0.045\textwidth-0.5\height\relax}{\textbf{Unlearn}} &
\includegraphics[width=0.09\textwidth]{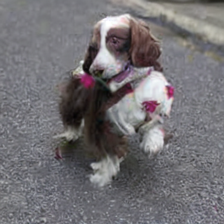} & 
\includegraphics[width=0.09\textwidth]{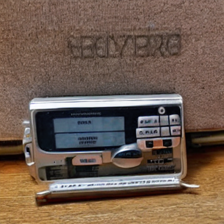} & 
\includegraphics[width=0.09\textwidth]{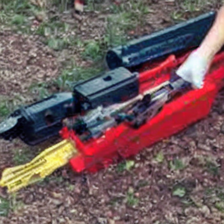} & 
\includegraphics[width=0.09\textwidth]{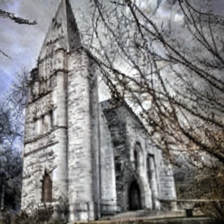} & 
\includegraphics[width=0.09\textwidth]{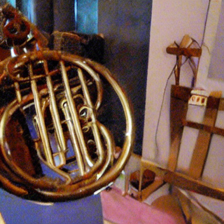} & 
\includegraphics[width=0.09\textwidth]{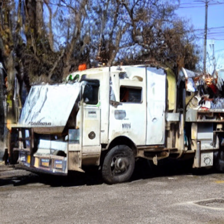} & 
\includegraphics[width=0.09\textwidth]{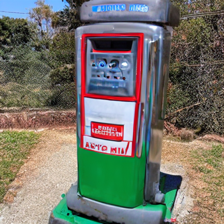} & 
\includegraphics[width=0.09\textwidth]{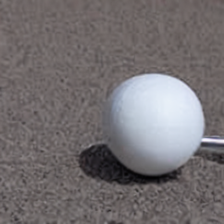} & 
\includegraphics[width=0.09\textwidth]{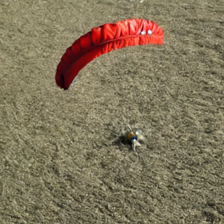} & 
\includegraphics[width=0.09\textwidth]{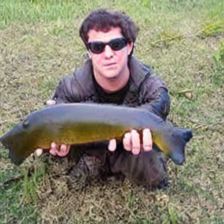} \\
\midrule
\raisebox{\dimexpr0.045\textwidth-0.5\height\relax}{\textbf{Retain}} &
\includegraphics[width=0.09\textwidth]{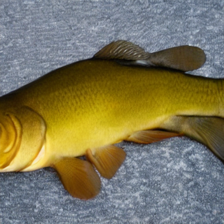} & 
\includegraphics[width=0.09\textwidth]{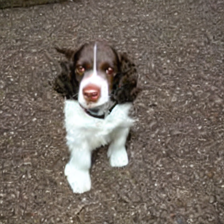} & 
\includegraphics[width=0.09\textwidth]{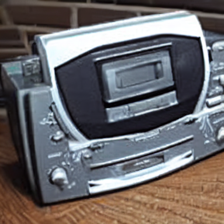} & 
\includegraphics[width=0.09\textwidth]{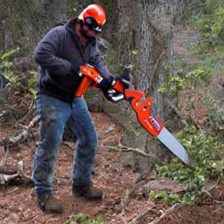} & 
\includegraphics[width=0.09\textwidth]{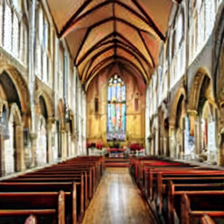} & 
\includegraphics[width=0.09\textwidth]{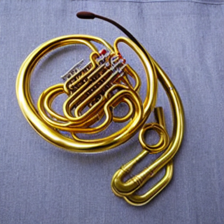} & 
\includegraphics[width=0.09\textwidth]{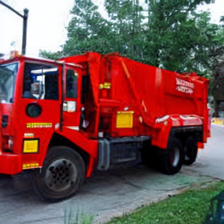} & 
\includegraphics[width=0.09\textwidth]{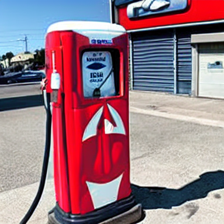} & 
\includegraphics[width=0.09\textwidth]{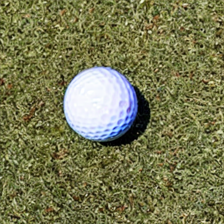} & 
\includegraphics[width=0.09\textwidth]{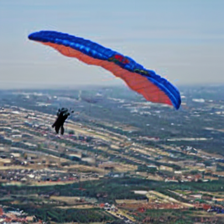} \\
\bottomrule
\end{tabular}}
\vspace{-0.35cm}
\caption{Qualitative results of unlearning and retaining images across Imagenette classes using Stable Diffusion. The top row shows generated samples for each class when it is targeted for unlearning, while the bottom row displays samples for the same classes when they are retained.
}
\label{fig:sd_main}
\vspace{-0.5cm}
\end{figure*}

While the method requires SVD computation during training, we note that the method has minimal overhead. For U-Net architectures, 
\begin{wraptable}[13]{R}{0.63\linewidth}
\vspace{-0.8cm}
\centering
\small
\setlength{\tabcolsep}{6pt}
\caption{Class-wise unlearning of image generation on Imagenette with SD. 
}
\label{tab:sd_imagenette}
\resizebox{\linewidth}{!}{
\begin{tabular}{l|cc|cc|>{\columncolor{gray!20}}c>{\columncolor{gray!20}}c}
\specialrule{.15em}{.05em}{.05em}
\multicolumn{1}{l|}{\textbf{Forget Class}}
& \multicolumn{2}{c|}{\textbf{ESD}} 
& \multicolumn{2}{c|}{\textbf{SalUn}} 
& \multicolumn{2}{c}{\cellcolor{gray!20}\textbf{+ SSU}} \\
& FA ($\downarrow$) & FID ($\downarrow$)
& FA ($\downarrow$) & FID ($\downarrow$)
& \cellcolor{gray!20}FA ($\downarrow$) & \cellcolor{gray!20}FID ($\downarrow$) \\
\midrule
Tench            & 0.00 & 2.37 & 0.00 & 0.91 & 0.00 & 0.76 \\
English Springer & 0.00 & 1.70 & 0.00 & 0.92 & 0.00 & 1.17 \\
Cassette Player  & 0.00 & 1.34 & 0.67 & 1.30 & 0.67 & 1.16 \\
Chain Saw        & 0.00 & 1.48 & 0.00 & 1.38 & 0.00 & 1.55 \\
Church           & 2.33 & 2.56 & 0.00 & 1.73 & 0.00 & 1.65 \\
French Horn      & 0.00 & 1.61 & 0.00 & 1.36 & 0.00 & 1.14 \\
Garbage Truck    & 0.00 & 1.57 & 0.00 & 1.01 & 0.00 & 1.00 \\
Gas Pump         & 0.00 & 1.34 & 0.00 & 2.78 & 0.00 & 1.48 \\
Golf Ball        & 0.67 & 1.35 & 1.33 & 2.44 & 1.33 & 1.12 \\
Parachute        & 0.00 & 1.75 & 0.00 & 1.68 & 0.00 & 1.35 \\
\midrule
Average          & 0.30 & 1.71 & \bf 0.20 & 1.55 & \bf 0.20 & \bf 1.24 \\
\specialrule{.15em}{.05em}{.05em}
\end{tabular}
}

\end{wraptable}
gradients are reshaped to moderate-sized 2D matrices (typically $\leq 512 \times 512$) for SVD, which can be efficiently decomposed using optimized GPU implementations.

\vspace{-3pt}
In practice, our method takes approximately 40 seconds per 100 training steps on one NVIDIA L40S, nearly identical to SalUn, while providing better unlearning through spectral filtering.\vspace{3pt}\\
\myparagraph{Results on SD.} 
\tabref{tab:sd_imagenette} shows class-wise unlearning 
performance on Stable Diffusion~\cite{sd} with the Imagenette dataset~\cite{imagenette}. We generate 300 images per class for the computation of both FA and FID. All methods achieve strong forgetting performance with average FA below 0.30\%, confirming effective removal of target concepts. 
In terms of generation quality, ESD achieves an average FID of 1.71, while SalUn improves this to 1.55. Our method further advances the state-of-the-art with an average FID of 1.24, achieving the best overall performance.
SSU demonstrates consistent improvements across most classes, 
with larger gains on Tench, French Horn, Gas Pump, and Golf Ball.
Moreover, SSU achieves slightly better forgetting with an average FA of 0.20\% compared to 0.30\% for ESD, matching SalUn's 0.20\%. These results validate the generalizability of our approach across different diffusion architectures. 

\figref{fig:sd_main} shows qualitative results for class-wise unlearning on Stable Diffusion with the Imagenette dataset.  We observe that SSU effectively removes the model's ability to generate images corresponding to the unlearned classes, as evidenced by the lack of recognizable features in those classes, \eg, a tench turns into a dog. At the same time, the generations of retained classes show high fidelity and remain semantically accurate. This demonstrates that our approach preserves generative quality for non-targeted concepts. These results visually verify the effectiveness of our method in selectively unlearning specific classes while maintaining overall model utility.

\vspace{-0.1cm}
\subsection{Random Subset Unlearning in LLMs}
\label{sec:unlearn_llm}

\myparagraph{Experimental setup and evaluation.}
Following the setup by~\citet{yuan2025closer}, we evaluate SSU on LLM unlearning using the TOFU benchmark~\cite{tofu} with Llama2-7B~\cite{llama2}. The TOFU benchmark consists of three scenarios: forget01, forget05, and forget10, corresponding to forgetting 1\%, 5\%, and 10\% of the training data, respectively. We consider two types of unlearning tasks: untargeted unlearning, including gradient-based methods (GA+GD, GA+KL), preference-based methods (NPO+GD, NPO+KL), and knowledge manipulation methods (ME+GD); and targeted unlearning, where we evaluate on preference-based methods (DPO+GD, DPO+KL) and knowledge manipulation methods (IDK+GD, IDK+AP). The baselines are adopted from~\citet{yuan2025closer}. SSU is built on top of the best-performing baseline for each task. 

As in prior work~\cite{yuan2025closer}, we evaluate performance using two aggregated metrics: Model Utility (MU) and Forget Efficacy (FE). These metrics combine multiple complementary signals, including ROUGE, prediction probability, truth ratio, token entropy, semantic similarity, and entailment score, to jointly assess lexical accuracy, semantic consistency, generation quality, and factual correctness. MU is computed on the retain set using the harmonic mean to measure overall utility preservation, while FE is computed on the forget set as one minus the arithmetic mean (excluding token entropy) to quantify the strength of forgetting. We further report their average as a single summary score. Implementation details are provided in Appx.~\ref{sec:appendix_llm}.

\myparagraph{Results.} 
\tabref{tab:llm_tofu} presents random subset unlearning on Llama2-7B across three forgetting scenarios. We apply SSU on top of the best baseline methods for each unlearning type. For untargeted unlearning, ME+GD emerges as the strongest baseline, significantly outperforming gradient-based (GA) and preference-based (NPO) alternatives. Applying SSU on top of ME+GD further enhances performance. SSU improves model utility for forget01 from 0.7245 to 0.7477 and forget10 from 0.7312 to 0.7480 while simultaneously improving forget efficacy across all scenarios, demonstrating superior balance between preserving model capabilities and effective unlearning. 

For targeted unlearning, IDK+AP serves as the best baseline with consistent performance. Building upon IDK+AP, SSU achieves the best results with average scores of 0.7856, 0.7562, and 0.7475 for the three scenarios. The improvements are particularly noticeable for forget01, where SSU enhances both model utility and forget efficacy. These results demonstrate that SSU generalizes effectively to large language models, consistently improving state-of-the-art unlearning methods for both untargeted and targeted scenarios across different forgetting ratios.
\begin{table*}[t]
\vspace{-0.4cm}
\centering
\small
\setlength{\tabcolsep}{6pt}
\caption{Results of different unlearning methods on the TOFU benchmark with Llama2-7B. MU and FE represent Model Utility and Forget Efficacy, respectively, and we indicate the best results in bold.}
\label{tab:llm_tofu}
\renewcommand{\arraystretch}{1.15}
\resizebox{\textwidth}{!}{
\begin{tabular}{lc|cccccc|ccccc}
\specialrule{.15em}{.05em}{.05em}
\multirow{2}{*}{\textbf{Dataset}} &
\multirow{2}{*}{\textbf{Metric}} &
\multicolumn{11}{c}{\textbf{Method}} \\
\cmidrule(lr){3-13}
& & GA+GD & GA+KL & NPO+GD & NPO+KL & ME+GD & \cellcolor{gray!20}+ SSU & DPO+GD & DPO+KL & IDK+GD & IDK+AP & \cellcolor{gray!20}+ SSU \\
\midrule
\multirow{3}{*}{\textbf{forget01}} 
& MU   & 0.6671 & 0.6385 & 0.6402 & \underline{0.7404} & 0.7245 & \cellcolor{gray!20}\bf 0.7477 & 0.7554 & \underline{0.7601} & 0.6704 & 0.7579 & \cellcolor{gray!20}\bf 0.7769 \\
& FE   & 0.5935 & 0.6028 & 0.6137 & 0.4943 & \underline{0.9156} & \cellcolor{gray!20}\bf 0.9599 & 0.5260 & 0.3013 & \underline{0.7700} & 0.7625 & \cellcolor{gray!20}\bf 0.7943 \\
& Avg. & 0.6303 & 0.6206 & 0.6269 & 0.6174 & \underline{0.8201} & \cellcolor{gray!20}\bf 0.8538 & 0.6407 & 0.5307 & 0.7202 & \underline{0.7602} & \cellcolor{gray!20}\bf 0.7856 \\
\midrule
\multirow{3}{*}{\textbf{forget05}}
& MU   & 0.2913 & 0.0000 & 0.5718 & 0.5470 & \bf 0.7521 & \cellcolor{gray!20}\underline{0.7277} & 0.0000 & 0.4536 & 0.0000 & \bf 0.7522 & \cellcolor{gray!20}\underline{0.7515} \\
& FE   & 0.9135 & 0.8927 & 0.6967 & 0.6192 & \underline{0.9262} & \cellcolor{gray!20}\bf 0.9407 & \bf 0.8242 & 0.7831 & \underline{0.7948} & 0.7480 & \cellcolor{gray!20}0.7608 \\
& Avg. & 0.6024 & 0.4463 & 0.6342 & 0.5831 & \bf 0.8391 & \cellcolor{gray!20}\underline{0.8342} & 0.4121 & 0.6183 & 0.3974 & \underline{0.7501} & \cellcolor{gray!20}\bf 0.7562 \\
\midrule
\multirow{3}{*}{\textbf{forget10}}
& MU   & 0.5001 & 0.0000 & 0.5660 & 0.4904 & \underline{0.7312} & \cellcolor{gray!20}\bf 0.7480 & 0.0000 & 0.0000 & 0.0527 & \underline{0.7444} & \cellcolor{gray!20}\bf 0.7463 \\
& FE   & 0.9449 & 0.9484 & 0.7530 & 0.7449 & \underline{0.9505} & \cellcolor{gray!20}\bf 0.9557 & \underline{0.8043} & \bf 0.8346 & 0.7602 & 0.7432 & \cellcolor{gray!20}0.7487 \\
& Avg. & 0.7225 & 0.4742 & 0.6595 & 0.6177 & \underline{0.8409} & \cellcolor{gray!20}\bf 0.8519 & 0.4022 & 0.4173 & 0.4065 & \underline{0.7438} & \cellcolor{gray!20}\bf 0.7475 \\
\specialrule{.15em}{.05em}{.05em}
\end{tabular}
}
\vspace{-.6cm}
\end{table*}

\begin{table}[t]
\begin{minipage}[t]{0.56\linewidth}
\centering
\small
\caption{Ablating the top singular value ratio when
randomly unlearning 10\% of the CIFAR-10 data.}
\label{tab:ablation_ratio}
\resizebox{\linewidth}{!}{
\begin{tabular}{c|ccccc}
\specialrule{.15em}{.05em}{.05em}
{Keep ratio} & {FA} & {RA} & {TA} & {MIA} & {Avg.\ Gap $\downarrow$} \\
\midrule
{0.1} & {98.27$\pm$0.64} & {99.53$\pm$0.31} & {93.22$\pm$0.45} & {11.10$\pm$0.93} & {1.86$\pm$0.11} \\  {0.3} & {97.45$\pm$0.57} & {99.34$\pm$0.25} & {93.04$\pm$0.43} & {12.81$\pm$0.37} & {1.32$\pm$0.24} \\  {0.5} & {97.00$\pm$0.61} & {99.29$\pm$0.26} & {93.08$\pm$0.38} & {13.34$\pm$0.99} & {1.23$\pm$0.25} \\  \rowcolor{gray!20}
{0.7} & {97.06$\pm$0.56} & {99.38$\pm$0.21} & {93.26$\pm$0.38} & {13.38$\pm$0.67} & {\textbf{1.10$\pm$0.20}} \\  {0.9} & {97.21$\pm$0.61} & {99.54$\pm$0.15} & {93.29$\pm$0.40} & {13.81$\pm$0.68} & {1.15$\pm$0.21} \\  \specialrule{.15em}{.05em}{.05em}
\end{tabular}
}

\end{minipage}\hfill
\begin{minipage}[t]{0.40\linewidth}
\centering
\small
\setlength{\tabcolsep}{6pt}
\caption{Ablating the modules to which SVD is applied
for class-wise unlearning with Stable Diffusion.}
\label{tab:ablation_module}
\begin{tabular}{c|cccc}
\specialrule{.15em}{.05em}{.05em}
 & Conv & MLP & Full & XAttn \\
\midrule
FA $\downarrow$  & 0.17 & 0.23 & 0.23 & 0.20 \\
FID $\downarrow$        & 1.29 & 1.29 & 1.28 & 1.24 \\
\specialrule{.15em}{.05em}{.05em}
\end{tabular}
\vspace{-0.2cm}

\end{minipage}
\vspace{-.5cm}
\end{table}

\vspace{-0.1cm}
\subsection{Ablation Studies}
\myparagraph{Effect of the keep ratio.}
To investigate the impact of the keep ratio on unlearning performance, we conduct an ablation study by varying the ratio parameter in our method for random subset unlearning on CIFAR-10. As in \tabref{tab:ablation_ratio}, the method is relatively robust with respect to the keep ratio except in extreme cases such as $0.1$. Our default setting of $\text{ratio}=0.7$ achieves the best balance across different metrics.

\myparagraph{Impact of SVD-applied weight selection.}
We ablate the effect of applying SVD to different subsets of weights in the U-Net architecture. \tabref{tab:ablation_module} reports results on Stable Diffusion with Imagenette, comparing SVD applied to all layers, only cross-attention layers, only convolutional layers, and only MLP layers. We observe that restricting SVD to cross-attention layers achieves the best trade-off between forgetting and generation quality.
In contrast, applying SVD to all layers slightly degrades generation quality, likely because it perturbs low-level visual features and global denoising dynamics that are less directly related to semantic concepts. This suggests that effective unlearning requires targeted modification of concept-carrying components. These results motivate our design choice of selectively applying SVD to cross-attention layers, which maximizes unlearning efficacy while preserving overall generative fidelity.

\begingroup
\let\cite\citet
\vspace{-0.3cm}
\section{Related Work}
\vspace{-0.1cm}
\myparagraph{Forgetting-retention trade-off and spectral structure.}
The forgetting-retention trade-off is a central consideration in machine unlearning \citep{sekhari2021remember}. As updates that remove forget-set influence can degrade retained performance, unlearning inherits the same interference phenomenon studied as catastrophic forgetting in continual learning~\citep{li2019learn}. Many methods, therefore, use SVD-based projection to project the unlearning update onto a subspace orthogonal to an estimated retain subspace~\citep{wang2025gru, lin2024gdrgma, biswas2025cure, fang2025alphaedit, chen2024machine}, following parameter-isolation strategies from continual learning \citep{farajtabar2020orthogonal, bennani2020generalisation, saha2021gradient}.

SSU differs from these approaches in both object and mechanism. It operates directly on the singular basis of the unlearning gradient and suppresses weak spectral components. This yields an optimizer-agnostic mechanism that targets forget-retain interaction directly, rather than constraining updates solely through an estimated retain subspace. Related work also studies forget-retain conflict via Euclidean gradient alignment between the forgetting and retaining gradients \citep{patel2025learning, wang2025rethinking, asif2026ofmu}, whereas SSU analyzes their interaction in the spectral basis of the unlearning gradient and uses the resulting spectrum as a structured proxy for interference. In addition, \cite{sendera2025semu} applies SVD-based low-rank parameterization for efficient updates, while SSU uses singular values as a saliency signal for thresholding.

\myparagraph{Saliency-based MU.} 
Existing saliency-based unlearning selectively updates a subset of parameters deemed most responsible for forget-set behavior. Motivated by evidence that model sparsity improves unlearning \citep{savani2025antidistillation}, SalUn \citep{fan2024salun} computes a coordinate-wise saliency map from the forget-set gradient evaluated with the pretrained model, and thresholds low-magnitude coordinates to sparsify unlearning updates. Building on SalUn, \cite{ding2025understanding} further argues via a fine-tuning analysis that saliency should instead be derived from retained data to better preserve overlapping features and mitigate the forgetting-retention trade-off. 

Unlike coordinate masking, SSU defines spectral saliency using the singular values of the unlearning gradient, selecting dominant matrix-valued update directions that reflect signals from both forgetting and retention. \cite{huang2024unified} develops a unified view of saliency by decomposing the unlearning objective into a forgetting term, a retention term, and an explicit weight-saliency matrix that modulates the unlearning direction. Overall, while saliency-driven sparsification is empirically effective, it remains largely a heuristic. Our analysis offers a complementary theoretical explanation for why saliency-based masking can improve the forgetting-retention trade-off.

\myparagraph{General MU}
has been explored through a range of optimization-based approaches, including gradient ascent methods that increase the forget-set loss~\citep{thudi2022unrolling, neel2021descent}, influence-function-based approximations to leave-one-out retraining~\citep{izzo2021approximate, koh2017understanding}, and Fisher-information-based selective updates~\citep{golatkar2020eternal, becker2022evaluating}. More recently, practical approximate unlearning has increasingly relied on fine-tuning \citep{warnecke2021machine, golatkar2020eternal, kurmanji2023towards}, particularly for large models where full retraining is prohibitive. These techniques have been developed across modalities, including image classification models \citep{jia2023model, fan2024salun}, generative image models such as diffusion models \citep{gandikota2023erasing, zhang2024forget, heng2023selective, wu2025erasing}, and LLMs \citep{zhang2024negative, yuan2025closer, liu2025rethinking}. 

Despite their differences, many existing methods can be expressed under a common template that jointly optimizes a forgetting loss and a retention or utility-preserving loss \citep{fan2024salun, yuan2025closer, zhong2025dualoptim}. SSU can be used as a plug-in enhancement to methods of this structure, as we have demonstrated in~\secref{sec:exp}.

\vspace{0.5pt}

\endgroup
\section{Conclusion}
Motivated by a spectral view of gradient-based optimization, we proposed Spectral Saliency Unlearning (SSU), which applies singular-value thresholding to suppress weak spectral components of the unlearning gradient and can be used as a drop-in enhancement for gradient-based unlearning pipelines. We provided a theoretical justification for SSU from the perspective of the forgetting-retention trade-off, characterizing how weak directions relate to forget-retain interference and how suppressing them can mitigate utility degradation. We further extended this perspective to coordinate-wise thresholding, offering a nontrivial explanation for the empirical effectiveness of SalUn-style masking. Empirically, we evaluated SSU across diverse unlearning settings, including image classification, diffusion-model unlearning, and large language model unlearning. Across these settings, SSU consistently demonstrates its effectiveness in unlearning and improves the forgetting-utility trade-off.

{\small
\bibliographystyle{ieeenat_fullname}
\bibliography{ref,ref_imma}
}

\appendix
\clearpage
\renewcommand{\theHproposition}{appendix.\theHsection.\arabic{proposition}}
\appendix

{\bf \Large Appendix}

\noindent The appendix is organized as follows:
\begin{itemize}[noitemsep,leftmargin=*,topsep=0em]
    \item In~\secref{sec:proof}, we provide the formal assumptions, statements, and complete proofs for the propositions stated in the main paper.
    \item In~\secref{sec:supp_exp}, we provide additional experiment details. The code will be open-sourced upon the acceptance of this paper.
    \item In~\secref{sec:supp_results}, we provide additional image generation results after unlearning with SSU.
\end{itemize}

\section{Proof of the Theoretical Justification} \label{sec:proof}

\subsection{Formal Statement and Proof of Proposition \ref{prop:conflict}} \label{sec:appendix_ssu_1}

For the unlearning objective $\cL_u = \cL_f + \cL_r$ composed of a forgetting loss $\cL_f$ and a retaining loss $\cL_r$, we have for their gradients $\mG_u = \mG_f + \mG_r$. The SVD of $\mG_u$ is given by $\mG_u = \mU \diag\pa{\bm{\sigma}}\mV = \sum_{i=1}^m \sigma_i \vu_i \vv_i^\top$. Let $r_e=\mathrm{rank}\pa{G_u}$. Then we have $\forall i \in [r_e]$, $\sigma_i > 0$ and $\forall i \in \set{r_e+1, \cdots, m}$, $\sigma_i = 0$. In practice, in the context of effective rank, we let the former represent significant non-zero singular values and the latter include those that are approximately zero. We have compact SVD $\mG_u = \mU_e\mathrm{diag}\pa{\bm{\sigma}_e}{\mV_e}^\top$ where $\mU_e \in \R^{m \times r_e}$ and $\mV_e \in \R^{n \times r_e}$.

\begin{assumption} [Smoothness] \label{asm:smooth} For constants $\beta_f, \beta_r > 0$, the forgetting loss $\cL_f$ and the retaining loss $\cL_r$ are smooth, i.e., for $\mW$, $\mW'$,
\begin{align*}
    \abs{\cL_f\pa{\mW'} - \cL_f\pa{\mW} - \inner{\mG_f\pa{\mW}}{\mW' - \mW}} \leq \frac{\beta_f}{2}\norm{\mW' - \mW}^2, \\
    \abs{\cL_r\pa{\mW'} - \cL_r\pa{\mW} - \inner{\mG_r\pa{\mW}}{\mW' - \mW}} \leq \frac{\beta_r}{2}\norm{\mW' - \mW}^2.
\end{align*}
\end{assumption}

\defalignment*

\begin{lemma} \label{lem:vec_inner}
    If $a\pa{\vv_i} < -\frac{\min\set{\norm{\mG_f\vv_i}, \norm{\mG_r\vv_i}}}{\max\set{\norm{\mG_f\vv_i}, \norm{\mG_r\vv_i}}}$, then for $f_i = \inner{\mG_f}{\vu_i\vv_i^\top}$, $r_i = \inner{\mG_r}{\vu_i\vv_i^\top}$, we have $f_ir_i < 0$.
\end{lemma}
\begin{proof}
    By definition,
\begin{align*}
    f_i = \inner{\mG_f}{\vu_i\vv_i^\top} = \vu_i^\top \mG_f \vv_i, && r_i = \inner{\mG_r}{\vu_i\vv_i^\top} = \vu_i^\top \mG_r \vv_i 
\end{align*}
Since $\mG_f \vv_i + \mG_r \vv_i = \mG_u \vv_i = \sum_{j=1}^{r_e} \sigma_j \vu_j \vv_j^\top \vv_i = \sigma_i \vu_i$, we know $\vu_i = \frac{\mG_f \vv_i + \mG_r \vv_i}{\sigma_i}$. Therefore,
\begin{align*}
    f_i = \vu_i^\top \mG_f \vv_i = \frac{\pa{\mG_f \vv_i + \mG_r \vv_i}^\top \mG_f \vv_i}{\sigma_i} = \frac{\norm{\mG_f \vv_i}_2^2 + \inner{\mG_r \vv_i}{\mG_f \vv_i}}{\sigma_i}.
\end{align*}
Similarly, $r_i = \frac{\norm{\mG_r \vv_i}_2^2 + \inner{\mG_r \vv_i}{\mG_f \vv_i}}{\sigma_i}$. 

If $a\pa{\vv_i} = \frac{\inner{\mG_f\vv_i}{\mG_r\vv_i}}{\norm{\mG_f\vv_i}_2\norm{\mG_r\vv_i}_2} < -\frac{\min\set{\norm{\mG_f\vv_i}, \norm{\mG_r\vv_i}}}{\max\set{\norm{\mG_f\vv_i}, \norm{\mG_r\vv_i}}}$, without loss of generality, we assume $\norm{\mG_f\vv_i} > \norm{\mG_r\vv_i}$, then $\frac{\inner{\mG_f\vv_i}{\mG_r\vv_i}}{\norm{\mG_f\vv_i}_2\norm{\mG_r\vv_i}_2} < -\frac{\norm{\mG_r\vv_i}}{\norm{\mG_f\vv_i}}$ yields $\inner{\mG_f\vv_i}{\mG_r\vv_i} < - \norm{\mG_r\vv_i}^2$. Also, we have by Cauchy-Schwarz inequality,
\begin{align*}
    \inner{\mG_f\vv_i}{\mG_r\vv_i} \geq - \norm{\mG_f\vv_i}_2\norm{\mG_r\vv_i}_2 > - \norm{\mG_f\vv_i}_2^2.
\end{align*}
Therefore, we know $\inner{\mG_f\vv_i}{\mG_r\vv_i} + \norm{\mG_r\vv_i}^2 < 0$ and $\inner{\mG_f\vv_i}{\mG_r\vv_i} + \norm{\mG_f\vv_i}_2^2 > 0$. As a result, we have
\begin{align*}
    f_i r_i = \frac{\pa{\norm{\mG_f \vv_i}_2^2+\inner{\mG_r \vv_i}{\mG_f \vv_i}}\pa{\norm{\mG_r \vv_i}_2^2+\inner{\mG_r \vv_i}{\mG_f \vv_i}}}{\sigma_i^2} < 0.
\end{align*}
\end{proof}

\begin{lemma} \label{lem:conflict}
    For gradient $\mG_u = \eta \sum_{i=1}^{r_e} \sigma_i \vu_i\vv_i^\top$, the update in direction $\vv_i$ is given by $\mW' = \mW - \eta \sigma_i \vu_i\vv_i^\top$. For $f_i = \inner{\mG_f}{\vu_i\vv_i^\top}$, $r_i = \inner{\mG_r}{\vu_i\vv_i^\top}$, $\eta < \min_i \set{\frac{2\abs{f_i}}{\beta_f\sigma_i}, \frac{2\abs{r_i}}{\beta_r\sigma_i}}$, forgetting loss $\cL_f$, and retaining loss $\cL_r$, we have
    \begin{itemize}
        \item [{\bf (a)}] $\cL_f\pa{\mW'} - \cL_f\pa{\mW} < 0$ implies $f_i > 0$, and $f_i < 0$ implies $\cL_f\pa{\mW'} - \cL_f\pa{\mW} > 0$.
        \item [{\bf (b)}] $\cL_r\pa{\mW'} - \cL_r\pa{\mW} < 0$ implies $r_i > 0$, and $r_i < 0$ implies $\cL_r\pa{\mW'} - \cL_r\pa{\mW} > 0$.
    \end{itemize}
\end{lemma}
\begin{proof}
    (a) We prove by contradiction. Suppose when $\cL_f\pa{\mW'} - \cL_f\pa{\mW} < 0$, $f_i \leq 0$. For the update $\mW' = \mW + \Delta\mW$ where $\Delta \mW = - \eta \sigma_i \vu_i\vv_i^\top$, we have $\norm{\Delta\mW} = \eta \sigma_i$. By Assumption \ref{asm:smooth}, we have
\begin{align*}
    \cL_f\pa{\mW'} - \cL_f\pa{\mW} &\geq \inner{\mG_f}{\Delta\mW} - \frac{\beta_f}{2}\norm{\Delta\mW}^2 \\
    &= - \eta \sigma_i \inner{\mG_f}{\vu_i\vv_i^\top} - \frac{\beta_f}{2} \eta^2\sigma_i^2 \\
    &= - \eta\sigma_i f_i - \frac{\beta_f}{2} \eta^2\sigma_i^2 \\
    &> 0,
\end{align*}
where the last inequality follows from $\eta < \frac{2\abs{f_i}}{\beta_f\sigma_i}$. This contradicts the condition $\cL_f\pa{\mW'} - \cL_f\pa{\mW} < 0$. Therefore, when $\cL_f\pa{\mW'} - \cL_f\pa{\mW} < 0$, we must have $f_i > 0$. In addition, from the derivation above, we know $f_i = \inner{\mG_f}{\vu_i\vv_i^\top} \leq 0$ implies $\cL_f\pa{\mW'} - \cL_f\pa{\mW} > 0$.

(b) Suppose when $\cL_r\pa{\mW'} - \cL_r\pa{\mW} < 0$, $r_i \leq 0$. By Assumption \ref{asm:smooth}, we have
\begin{align*}
    \cL_r\pa{\mW'} - \cL_r\pa{\mW} &\geq \inner{\mG_r}{\Delta\mW} - \frac{\beta_r}{2}\norm{\Delta\mW}^2 \\
    &= - \eta \sigma_i \inner{\mG_r}{\vu_i\vv_i^\top} - \frac{\beta_r}{2} \eta^2\sigma_i^2 \\
    &= - \eta\sigma_i r_i - \frac{\beta_r}{2} \eta^2\sigma_i^2 \\
    &> 0,
\end{align*}
where the last inequality follows from $\eta < \frac{2\abs{r_i}}{\beta_r\sigma_i}$. This contradicts the condition $\cL_r\pa{\mW'} - \cL_r\pa{\mW} < 0$. Therefore, when $\cL_r\pa{\mW'} - \cL_r\pa{\mW} < 0$, we must have $r_i > 0$. In addition, from the derivation above, we know $r_i = \inner{\mG_r}{\vu_i\vv_i^\top} \leq 0$ implies $\cL_r\pa{\mW'} - \cL_r\pa{\mW} > 0$.
\end{proof}

\begingroup
\renewcommand{\thesection}{\ref{sec:ssu_theory}}
\setcounter{proposition}{1}
\begin{proposition}
 For gradient of the unlearning objective $\mG_u = \eta \sum_{i=1}^{r_e} \sigma_i \vu_i\vv_i^\top$, the update in direction $\vv_i$ is given by $\mW' = \mW - \eta \sigma_i \vu_i\vv_i^\top$. Under Assumption \ref{asm:smooth}, if $a\pa{\vv_i} < -\frac{\min\set{\norm{\mG_f\vv_i}, \norm{\mG_r\vv_i}}}{\max\set{\norm{\mG_f\vv_i}, \norm{\mG_r\vv_i}}}$, and $\eta < \min_i \set{\frac{2\abs{f_i}}{\beta_f\sigma_i}, \frac{2\abs{r_i}}{\beta_r\sigma_i}}$ for $f_i = \inner{\mG_f}{\vu_i\vv_i^\top}$ and $r_i = \inner{\mG_r}{\vu_i\vv_i^\top}$, then we have {\bf (a)} if $\Delta \cL_f = \cL_f\pa{\mW'} - \cL_f\pa{\mW} < 0$ then $\Delta \cL_r = \cL_r\pa{\mW'} - \cL_r\pa{\mW} > 0$; {\bf (b)} if $\Delta \cL_r < 0$, then $\Delta \cL_f > 0$. That is, along any singular direction $\vv_i$ where the retain and forget gradients are in significant conflict, any update that improves one objective must necessarily worsen the other.
\end{proposition}
\endgroup
\begin{proof}
(a) If the update makes progress in forgetting, that is, $\Delta \cL_f = \cL_f\pa{\mW'} - \cL_f\pa{\mW} < 0$, then by Lemma \ref{lem:conflict} (a), we know $f_i > 0$. Also, from Lemma \ref{lem:vec_inner}, we know that if the retain and forget gradients are in significant conflict in the effective space, i.e., $a\pa{\vv_i} < -\frac{\min\set{\norm{\mG_f\vv_i}, \norm{\mG_r\vv_i}}}{\max\set{\norm{\mG_f\vv_i}, \norm{\mG_r\vv_i}}}$, we have $f_ir_i < 0$. Since $f_i > 0$, we must have $r_i < 0$. Then by Lemma \ref{lem:conflict} (b), we have $\cL_r\pa{\mW'} - \cL_r\pa{\mW} > 0$, meaning that the effort for retention is damaged.

(b) If the update makes progress in retention, that is, $\Delta \cL_r = \cL_r\pa{\mW'} - \cL_r\pa{\mW} < 0$, then by Lemma \ref{lem:conflict} (b), we know $r_i > 0$. Again, by Lemma \ref{lem:vec_inner}, we know if $a\pa{\vv_i} < -\frac{\min\set{\norm{\mG_f\vv_i}, \norm{\mG_r\vv_i}}}{\max\set{\norm{\mG_f\vv_i}, \norm{\mG_r\vv_i}}}$, we have $f_ir_i < 0$. Since $r_i > 0$, we must have $f_i < 0$. Then by Lemma \ref{lem:conflict} (a), we have $\cL_f\pa{\mW'} - \cL_f\pa{\mW} > 0$, meaning that unlearning is not making progress in forgetting.
\end{proof}

\subsection{Formal Statement and Proof of Proposition \ref{prop:small_singular}} \label{sec:appendix_ssu_2}

For compact SVD $\mG_u = \mU_e\mathrm{diag}\pa{\bm{\sigma}_e}{\mV_e}^\top$ where $\mU_e \in \R^{m \times r_e}$ and $\mV_e \in \R^{n \times r_e}$, we call the space spanned by $\mV_e$ the effective subspace, $\mathcal{E} = \mathrm{span}\pa{\mV_e}$, and the corresponding projection $\mP = \mV_e \mV_e^\top$. Consider the Gram matrix of $\mG_r$ and $\mG_f$: $\mM_r = \mP\mG_r^\top\mG_r\mP$ and $\mM_f = \mP\mG_f^\top\mG_f\mP$. Throughout this analysis, the eigenvalues of $\mM_f$ and $\mM_r$ refer
to their restriction to the effective subspace $\mathcal{E}$,
or equivalently, their eigen-decompositions are defined as $\mM_f = \sum_{j=1}^{r_e} \lambda_{f,j} \ve_{f,j} \ve_{f,j}^\top$, $\mM_r = \sum_{j=1}^{r_e} \lambda_{r,j} \ve_{r,j} \ve_{r,j}^\top$.

\begin{assumption} [Bounded Spectral Disparity] \label{asm:disparity}
    $\mM_r$ and $\mM_f$ exhibit bounded spectral disparity. Specifically, there exist constants $\alpha_f^- < 1 < \alpha_f^+, \alpha_r^- < 1 < \alpha_r^+$ but close enough to $1$ such that for the maximum, minimum, and average eigenvalues of $\mM_r, \mM_f$,
    \begin{align*}
         \alpha_f^- \lambda_f^{\tt max} \leq \lambda_f^{\tt avg} \leq \alpha_f^+\lambda_f^{\tt min}, && \alpha_r^- \lambda_r^{\tt max} \leq \lambda_r^{\tt avg} \leq \alpha_r^+\lambda_r^{\tt min}.
    \end{align*}
\end{assumption}
Bounded spectral disparity formalizes that the projected retain and forget signals are bulk-distributed on the effective space $\mathcal{E}$, and the spectrum within $\mathcal{E}$ is not so heavy-tailed that $\vv^\top \mM_r \vv$ or $\vv^\top \mM_f \vv$ varies by orders of magnitude across admissible directions. This is a natural regime for subspace-based unlearning, as $\mathcal{E}$ is defined from the joint gradient $\mG_u$, so it filters out directions where the update signal is negligible and where task-specific curvature can be extremely ill-conditioned. Empirically, the projected spectra within $\mathcal E$ are typically far less heavy-tailed than in the full parameter space.

\begin{lemma} \label{lem:concentration}
    Under Assumption \ref{asm:disparity}, for $\epsilon = \max\set{\frac{1}{\alpha_f^-}-1, 1-\frac{1}{\alpha_f^+}, \frac{1}{\alpha_r^-}-1, 1-\frac{1}{\alpha_r^+}}$, we have $\forall i \in [r_e]$, 
    \begin{align*}
        \pa{1-\epsilon} \lambda_f^{\tt avg} \leq \vv_i^\top \mM_f \vv_i \leq \pa{1+\epsilon} \lambda_f^{\tt avg}, && \pa{1-\epsilon} \lambda_r^{\tt avg} \leq \vv_i^\top \mM_r \vv_i \leq \pa{1+\epsilon} \lambda_r^{\tt avg}.
    \end{align*}
\end{lemma}
\begin{proof}
    Given that $\mM_f = \sum_{j=1}^{r_e} \lambda_{f,j} \ve_{f,j} \ve_{f,j}^\top$, we have
    \begin{align*}
        \vv_i^\top \mM_f \vv_i = \sum_{j=1}^{r_e} \lambda_{f,j} \pa{\inner{\ve_{f,j}}{\vv_i}}^2 \leq \lambda_f^{\tt max} \sum_{j=1}^{r_e} \pa{\inner{\ve_{f,j}}{\vv_i}}^2 \leq \lambda_f^{\tt max}
    \end{align*}
    as $\ve_{f,j}$ forms an orthonormal basis and $\vv_i$ is a unit vector so that $\sum_{j=1}^{r_e} \pa{\inner{\ve_{f,j}}{\vv_i}}^2 = 1$. Similarly, we have
    \begin{align*}
       \lambda_f^{\tt min} \leq \vv_i^\top \mM_f \vv_i \leq \lambda_f^{\tt max}, && \lambda_r^{\tt min} \leq \vv_i^\top \mM_r \vv_i \leq \lambda_r^{\tt max}.
    \end{align*}
    Given Assumption \ref{asm:disparity}, we know $\lambda_f^{\tt max} \leq \frac{1}{\alpha_f^-}\lambda_f^{\tt avg}$ and by definition, $\epsilon \geq \frac{1}{\alpha_f^-}-1$, therefore
    \begin{align*}
        \vv_i^\top \mM_f \vv_i \leq \lambda_f^{\tt max} \leq \frac{1}{\alpha_f^-}\lambda_f^{\tt avg} \leq \pa{1+\epsilon} \lambda_f^{\tt avg}.
    \end{align*}
    Similarly, we can show $ \vv_i^\top \mM_f \vv_i \geq \pa{1-\epsilon} \lambda_f^{\tt avg}$, and $\pa{1-\epsilon} \lambda_r^{\tt avg} \leq \vv_i^\top \mM_r \vv_i \leq \pa{1+\epsilon} \lambda_r^{\tt avg}$.
\end{proof}

\begingroup
\renewcommand{\thesection}{\ref{sec:ssu_theory}}
\setcounter{proposition}{2}
\begin{proposition}
{\bf (a)} If the direction carries a nontrivial forget/retain signal, i.e., $\norm{\mG_f \vv_i}+\norm{\mG_r \vv_i} \geq \xi$ for $\xi > 0$ and its singular value is small relative to this signal, i.e., $\sigma_i \leq \rho \xi$ for $\rho \in \pa{0, \frac{1}{\sqrt{2}}}$, then its forget-retain alignment satisfies $a\pa{\vv_i} \leq 2\rho^2-1 < 0$. \\
{\bf (b)} Under Assumption \ref{asm:disparity}, for two directions $\vv_i,\vv_j$, $i,j \in [r_e]$, with alignment scores $a\pa{\vv_i} \leq -\delta_i$ and $a\pa{\vv_j} \geq \delta_j$ where $\delta_i, \delta_j \in (0, 1]$, if the separation $\pa{\delta_i + \delta_j} > \frac{\epsilon}{1-\epsilon} \cdot \frac{s^2+1}{s}$ for $\epsilon = \max\set{\frac{1}{\alpha_f^-}-1, 1-\frac{1}{\alpha_f^+}, \frac{1}{\alpha_r^-}-1, 1-\frac{1}{\alpha_r^+}}$ where $s = \frac{\norm{\mG_r\mP}_F}{\norm{\mG_f\mP}_F}$, then $\sigma_i < \sigma_j$.
\end{proposition}
\endgroup
\begin{proof}
(a)  Given that $\mG_u = \sum_{i=1}^{r_e} \sigma_i \vu_i \vv_i^\top$, we have for the singular value
    \begin{align*}
        \sigma_i^2 &= \norm{\mG_u \vv_i}_2^2 \\
        &= \norm{\pa{\mG_f+\mG_r}\vv_i}_2^2 \\
        &= \norm{\mG_f\vv_i}_2^2 + \norm{\mG_r\vv_i}_2^2 + 2 \inner{\mG_f\vv_i}{\mG_r\vv_i} \\
        &= \norm{\mG_f\vv_i}_2^2 + \norm{\mG_r\vv_i}_2^2 + 2 \norm{\mG_f\vv_i}_2\norm{\mG_r\vv_i}_2 a\pa{\vv_i}
    \end{align*}
    where the last equality follows from Definition \ref{def:alignment}. Also, we know from the conditions that
    \begin{align*}
        \sigma_i \leq \rho \xi \leq \rho \pa{\norm{\mG_f \vv_i}+\norm{\mG_r \vv_i}}.
    \end{align*}
    As a result, we have $$\norm{\mG_f\vv_i}_2^2 + \norm{\mG_r\vv_i}_2^2 + 2 \norm{\mG_f\vv_i}_2\norm{\mG_r\vv_i}_2 a\pa{\vv_i} = \sigma_i^2 \leq  \rho^2 \pa{\norm{\mG_f \vv_i}+\norm{\mG_r \vv_i}}^2.$$
    Rearranging the terms yields
    $$2 \norm{\mG_f\vv_i}_2\norm{\mG_r\vv_i}_2 a\pa{\vv_i} \leq  \pa{\rho^2-1} \pa{\norm{\mG_f \vv_i}^2+\norm{\mG_r \vv_i}^2} + 2 \rho^2\norm{\mG_f\vv_i}_2\norm{\mG_r\vv_i}_2.$$
    Dividing both sides by $2 \norm{\mG_f\vv_i}_2\norm{\mG_r\vv_i}_2$,
    \begin{align*}
        a\pa{\vv_i} \leq \rho^2 + \frac{\rho^2-1}{2} \pa{\frac{\norm{\mG_f\vv_i}_2}{\norm{\mG_r \vv_i}} + \frac{\norm{\mG_r \vv_i}}{\norm{\mG_f\vv_i}_2}} \leq 2\rho^2 - 1,
    \end{align*}
    where the last inequality follows from the facts that $\pa{\frac{\norm{\mG_f\vv_i}_2}{\norm{\mG_r \vv_i}} + \frac{\norm{\mG_r \vv_i}}{\norm{\mG_f\vv_i}_2}} \geq 2$ and $\rho^2-1 \leq 0$. This completes the proof.

(b) From the proof of (a), we know
    \begin{align*}
        \sigma_i^2 = \norm{\mG_f\vv_i}_2^2 + \norm{\mG_r\vv_i}_2^2 + 2 \norm{\mG_f\vv_i}_2\norm{\mG_r\vv_i}_2 a\pa{\vv_i}.
    \end{align*}
    Furthermore,
    \begin{align*}
        \norm{\mG_f\vv_i}_2^2 &= \vv_i^\top \mG_f^\top \mG_f \vv_i \\
        &= \vv_i^\top \mP \mG_f^\top \mG_f \mP \vv_i \\
        &= \vv_i^\top \mM_f \vv_i
    \end{align*}
    where we used the fact that $\mP \vv_i = \vv_i$ as $\mP = \mV_e\mV_e^\top$ and $\vv_i$ is one column of $\mV_e$. Therefore, by Lemma \ref{lem:concentration} we know for $\norm{\mG_f\vv_i}_2^2$ and similarly for $\norm{\mG_r\vv_i}_2^2$
    \begin{align*}
        \pa{1-\epsilon} \lambda_f^{\tt avg} \leq \norm{\mG_f\vv_i}_2^2 \leq \pa{1+\epsilon} \lambda_f^{\tt avg}, && \pa{1-\epsilon} \lambda_r^{\tt avg} \leq \norm{\mG_r\vv_i}_2^2 \leq \pa{1+\epsilon} \lambda_r^{\tt avg}.
    \end{align*}
    As a result, we can have for direction $\vv_i$ with singular value $\sigma_i$ and conflict $a\pa{\vv_i} \leq - \delta_i < 0$,
    \begin{align*}
        \sigma_i^2 &= \norm{\mG_f\vv_i}_2^2 + \norm{\mG_r\vv_i}_2^2 + 2 \norm{\mG_f\vv_i}_2\norm{\mG_r\vv_i}_2 a\pa{\vv_i} \\
        &\leq \pa{1+\epsilon} \lambda_f^{\tt avg} + \pa{1+\epsilon} \lambda_r^{\tt avg} + 2\sqrt{\pa{1-\epsilon} \lambda_f^{\tt avg}\pa{1-\epsilon} \lambda_r^{\tt avg}} a\pa{\vv_i} \\
        &\leq \pa{1+\epsilon} \pa{\lambda_f^{\tt avg} + \lambda_r^{\tt avg}} - 2\pa{1-\epsilon}\delta_i \sqrt{\lambda_f^{\tt avg}\lambda_r^{\tt avg}}.
    \end{align*}
    And for direction $\vv_j$ with singular value $\sigma_j$ and conflict $a\pa{\vv_j} \geq \delta_j > 0$,
    \begin{align*}
        \sigma_j^2 &= \norm{\mG_f\vv_j}_2^2 + \norm{\mG_r\vv_j}_2^2 + 2 \norm{\mG_f\vv_j}_2\norm{\mG_r\vv_j}_2 a\pa{\vv_j} \\
        &\geq \pa{1-\epsilon} \lambda_f^{\tt avg} + \pa{1-\epsilon} \lambda_r^{\tt avg} + 2\sqrt{\pa{1-\epsilon} \lambda_f^{\tt avg}\pa{1-\epsilon} \lambda_r^{\tt avg}} a\pa{\vv_j} \\
        &\geq \pa{1-\epsilon} \pa{\lambda_f^{\tt avg} + \lambda_r^{\tt avg}+2\delta_j \sqrt{\lambda_f^{\tt avg}\lambda_r^{\tt avg}}}.
    \end{align*}
    Given the condition of a sufficient gap in conflict $\pa{\delta_i + \delta_j} > \frac{\epsilon}{1-\epsilon} \cdot \frac{s^2+1}{s}$ for $s = \frac{\norm{\mG_r\mP}_F}{\norm{\mG_f\mP}_F} = \sqrt{\frac{\lambda_r^{\tt avg}}{\lambda_f^{\tt avg}}}$, multiplying both sides by $2\pa{1-\epsilon}\sqrt{\lambda_f^{\tt avg}\lambda_r^{\tt avg}}$ yields
    \begin{align*}
        2\pa{1-\epsilon}\pa{\delta_i + \delta_j}\sqrt{\lambda_f^{\tt avg}\lambda_r^{\tt avg}} > 2\epsilon \pa{\lambda_f^{\tt avg} + \lambda_r^{\tt avg}},
    \end{align*}
    which, by rearranging the terms, is equivalent to 
    \begin{align*}
        \pa{1+\epsilon} \pa{\lambda_f^{\tt avg} + \lambda_r^{\tt avg}} - 2\pa{1-\epsilon}\delta_i \sqrt{\lambda_f^{\tt avg}\lambda_r^{\tt avg}} < \pa{1-\epsilon} \pa{\lambda_f^{\tt avg} + \lambda_r^{\tt avg}+2\delta_j \sqrt{\lambda_f^{\tt avg}\lambda_r^{\tt avg}}}.
    \end{align*}
    Therefore, we have 
    \begin{align*}
        \sigma_i^2 &\leq \pa{1+\epsilon} \pa{\lambda_f^{\tt avg} + \lambda_r^{\tt avg}} - 2\pa{1-\epsilon}\delta_i \sqrt{\lambda_f^{\tt avg}\lambda_r^{\tt avg}} \\
        &< \pa{1-\epsilon} \pa{\lambda_f^{\tt avg} + \lambda_r^{\tt avg}+2\delta_j \sqrt{\lambda_f^{\tt avg}\lambda_r^{\tt avg}}} \\
        &\leq \sigma_j^2.
    \end{align*}
\end{proof}

\subsection{Theoretical Justification for SalUn} \label{sec:appendix_salun}

\begin{assumption} [Smoothness] \label{asm:smooth_vec} For constants $\beta_f, \beta_r > 0$, the forgetting loss $\cL_f$ and the retaining loss $\cL_r$ are smooth, i.e., for $\vw$, $\vw'$,
\begin{align*}
    \abs{\cL_f\pa{\vw'} - \cL_f\pa{\vw} - \inner{\vg_f\pa{\vw}}{\vw' - \vw}} \leq \frac{\beta_f}{2}\norm{\vw' - \vw}^2, \\
    \abs{\cL_r\pa{\vw'} - \cL_r\pa{\vw} - \inner{\vg_r\pa{\vw}}{\vw' - \vw}} \leq \frac{\beta_r}{2}\norm{\vw' - \vw}^2.
\end{align*}
\end{assumption}

\propsalun*
\begin{proof}
    We first show that $g_{f,i} g_{r,i} < 0$. Assume that $g_{f,i} g_{r,i} \geq 0$, then we have
    $\abs{g_{u,i}} = \abs{g_{f,i}} + \abs{g_{r,i}} \geq \xi$, which contradicts the condition $\abs{g_{u,i}} \leq \rho \xi$ for $\rho < 1$. Therefore, $g_{f,i} g_{r,i} < 0$.

    Next, by Assumption \ref{asm:smooth_vec}, 
    \begin{align*}
         -\eta g_{f,i} g_{u,i}-\frac{\beta_f}{2}\eta^2 g_{u,i}^2 \leq \cL_f\pa{\vw'}-\cL_f\pa{\vw} \leq -\eta g_{f,i} g_{u,i}+\frac{\beta_f}{2}\eta^2 g_{u,i}^2,
    \end{align*}
    If $g_{f,i} g_{u,i} > 0$, then for $\Delta \cL_f = \cL_f\pa{\vw'}-\cL_f\pa{\vw}$,
    \begin{align*}
        \Delta \cL_f \leq -\eta g_{f,i} g_{u,i}+\frac{\beta_f}{2}\eta^2 g_{u,i}^2 = \eta \pa{\eta\frac{\beta_f}{2} g_{u,i}^2 - g_{f,i}g_{u,i}} < 0
    \end{align*}
    where the last line follows from $\eta < \frac{2\abs{g_{f,i}}}{\beta_f\abs{g_{u, i}}}$. Also, if $g_{f,i} g_{u,i} < 0$, then
    \begin{align*}
        \Delta \cL_f \geq -\eta g_{f,i} g_{u,i}-\frac{\beta_f}{2}\eta^2 g_{u,i}^2 = \eta \pa{\abs{g_{f,i}}\abs{g_{u,i}} - \eta\frac{\beta_f}{2} \abs{g_{u,i}}^2 } > 0
    \end{align*}
    where the last line again follows from $\eta < \frac{2\abs{g_{f,i}}}{\beta_f\abs{g_{u, i}}}$. As a result, we conclude that $\sign\pa{\Delta \cL_f} = -\sign \pa{g_{f,i} g_{u,i}}$. 
    
    Similarly, we can show for $\Delta \cL_r = \cL_r\pa{\vw'} - \cL_r \pa{\vw}$ that $\sign\pa{\Delta \cL_r} = -\sign \pa{g_{r,i} g_{u,i}}$. Therefore,
    $$\sign\pa{\Delta \cL_f \Delta \cL_r} = \sign \pa{g_{f,i} g_{u,i}} \sign \pa{g_{r,i} g_{u,i}} = \sign\pa{g_{f,i} g_{r,i} g_{u,i}^2} = \sign\pa{g_{f,i} g_{r,i}},$$
    which completes the proof.
\end{proof}

\section{Detailed Experiment Setup}\label{sec:supp_exp}

\subsection{Details of unlearning on image classification.}
\label{sec:appendix_image_cls}
\myparagraph{CIFAR-10.}
In the setting of \tabref{tab:cifar10_classification}, we randomly select 10\% of the training data to forget with the model choice of ResNet-18. This corresponds to 5,000 samples that should be unlearned. We use seeds 1-5 for selecting forget samples and initialization to ensure reproducibility. Each experiment is evaluated on four metrics: retain set accuracy with the 40,000 training samples (excluding 5,000 in the validation set), forget set accuracy with the 5,000 samples to be unlearned, test set accuracy (TA) with all 10,000 test samples, and the membership inference attack (MIA) with the forget set and test set.

During unlearning, we apply SVD on-the-fly. 
The unlearning process fine-tunes the model on the 40,000-sample retain set for 10 epochs with the learning rate of 0.0113 while applying spectral filtering according to the keep ratio $\gamma = 0.70$, which controls the proportion of singular values retained after reconstruction and was chosen by grid search, selecting the value that yields the best validation performance.

In the setting of \tabref{tab:cifar10_classification_50pct}, we randomly select 50\% of the training data to forget, corresponding to 25,000
samples to be unlearned. The keep ratio is set to be $\gamma = 0.70$ in consistency and learning rate 0.013. The seeds, models, and other settings follow the 10\% setting. 

\myparagraph{Tiny-ImageNet.}
We use the same ResNet-18 model on Tiny-ImageNet, which contains 200 classes. Images are upsampled from $64 \times 64$ to $224 \times 224$ before standard crop-and-flip augmentation. We randomly select 10,000 of the 100,000 training images as the forget set and report the test set accuracy on the 10,000 validation images. The keep ratio is set to $\gamma = 0.5$ and the learning rate to $0.003$. Results are averaged over three random seeds.

\subsection{Details of unlearning on image generation.}
\label{sec:appendix_image_gen}

\myparagraph{DDPM on CIFAR-10.}
We evaluate class-conditional unlearning on CIFAR-10 where we aim to remove the model's ability to generate images from a specific class. The diffusion model uses a U-Net architecture with 128 base channels, channel multipliers of $[1, 2, 2, 2]$, 2 residual blocks per resolution, and self-attention at resolution 16. The diffusion process uses a linear noise schedule with 1000 timesteps, $\beta_{\text{start}} = 0.0001$ and $\beta_{\text{end}} = 0.02$.

During unlearning, we fine-tune the diffusion model for 1000 iterations with batch size 128 using gradient ascent-based unlearning loss. For computational efficiency, we apply on-the-fly SVD reconstruction only to convolutional layers, while using pre-computed saliency coordinate masks for non-convolutional layers. For the gradient block of each convolutional layer, we compute its SVD $\mathbf{G} = \mathbf{U}\boldsymbol{\Sigma}\mathbf{V}^\top$ during training and apply eigenvalue-based filtering according to the keep ratio $\gamma = 0.50$, which controls the proportion of singular values retained. The unlearning uses Adam optimizer with learning rate $1 \times 10^{-5}$, remain alpha $\alpha_{\text{remain}} = 1 \times 10^{-3}$, and forget alpha $\alpha_{\text{forget}} = 1.0$. We obtained the hyperparameters by grid search, selecting the value that yields the best validation performance. For evaluation, we generate 5,000 images for each class to compute FID and FA.

\myparagraph{Stable Diffusion on Imagenette.} Stable Diffusion uses a latent diffusion model with a U-Net denoising network operating in the latent space of a pre-trained autoencoder. The U-Net has 320 base channels, channel multipliers of $[1, 2, 4, 4]$, 2 residual blocks per resolution, 8 attention heads, and transformer depth of 1 with context dimension 768 for CLIP text conditioning. The diffusion process uses a linear noise schedule with 1000 timesteps, $\beta_{\text{start}} = 0.00085$ and $\beta_{\text{end}} = 0.012$.

During unlearning, we fine-tune the model for 5 epochs with batch size 8 using gradient ascent-based unlearning loss with randomly assigned labels for the forget class. For computational efficiency, we apply on-the-fly SVD reconstruction only to cross-attention layers, while using pre-computed saliency coordinate masks for the other layers. For each convolutional layer, we compute its SVD $\mathbf{W} = \mathbf{U}\boldsymbol{\Sigma}\mathbf{V}^\top$ during training and apply eigenvalue-based filtering according to the keep ratio $\gamma = 0.50$, which controls the proportion of eigenvalues retained and was chosen by grid search, selecting the value that yields the best validation performance. The unlearning uses Adam optimizer with learning rate $1 \times 10^{-5}$ and classifier-free guidance scale of 7.5. For evaluation, we generate 300 images for each class to compute FID and FA.

\subsection{Details of unlearning on large language models.}
\label{sec:appendix_llm}

During unlearning, we fine-tune the model for 5 epochs with batch size 8 and gradient accumulation steps of 4, resulting in an effective batch size of 32. We experiment with two unlearning objectives: IDK+AP (I-Don't-Know with answer preservation) and ME+GD (mismatch entropy with gradient difference). For computational efficiency, we apply on-the-fly SVD reconstruction only to attention projection layers (Q, K, V), while using pre-computed saliency coordinate masks for feed-forward and other layers. For attention layers, we compute SVD during training and apply eigenvalue-based filtering according to the QKV mask ratio $\alpha_{\text{qkv}} = 0.20$, which controls the proportion of eigenvalues retained. For non-attention layers, we use coordinate masks with mask ratio of 0.10.

For IDK+AP, we use a learning rate $3 \times 10^{-5}$ with forget coefficient 1.0 and regularization coefficient 1.0. For ME+GD, we use a learning rate $2 \times 10^{-5}$ with forget coefficient 0.1 and regularization coefficient 1.0. The top eigenvalue ratio for both methods is 0.2, which was chosen by grid search, selecting the value that yields the best validation performance. All experiments use the AdamW optimizer with weight decay 0.01 and are trained with distributed data parallel across 2 GPUs.

\section{Additional generation results} \label{sec:supp_results}

\begin{figure*}[t]
\centering
\renewcommand{\arraystretch}{1.1}
\setlength{\tabcolsep}{2pt}
\caption{Examples of generated images using \texttt{SSU}. From the rows below, diagonal images represent the forgetting class, while non-diagonal images represent the remaining class.}
\label{fig:imagenette_grid_1}
\resizebox{\textwidth}{!}{\begin{tabular}{c|cccccccccc}
\toprule
\textbf{Unlearned} & \multicolumn{10}{c}{\textbf{Prompt class}} \\
\textbf{class} & Tench & springer & Cassette & Saw & Church & French horn & truck & Gas pump & Golf ball & Parachute \\
\midrule
Tench            & \includegraphics[width=0.07\textwidth]{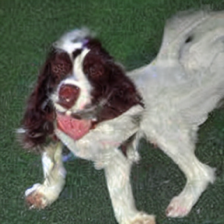} & \includegraphics[width=0.07\textwidth]{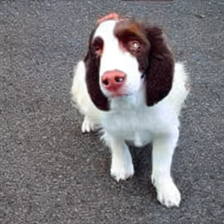} & \includegraphics[width=0.07\textwidth]{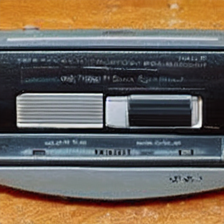} & \includegraphics[width=0.07\textwidth]{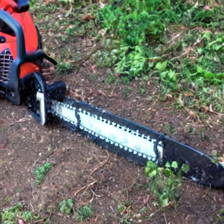} & \includegraphics[width=0.07\textwidth]{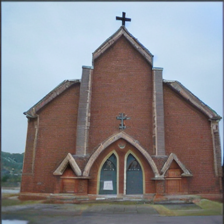} & \includegraphics[width=0.07\textwidth]{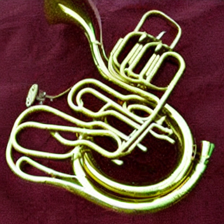} & \includegraphics[width=0.07\textwidth]{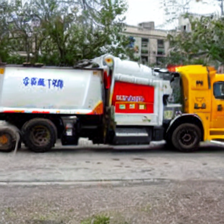} & \includegraphics[width=0.07\textwidth]{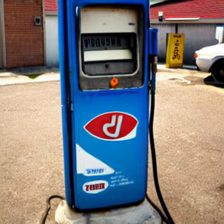} & \includegraphics[width=0.07\textwidth]{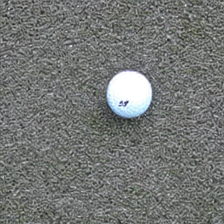} & \includegraphics[width=0.07\textwidth]{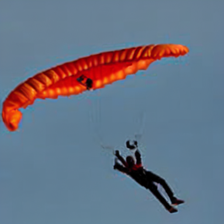} \\
English springer & \includegraphics[width=0.07\textwidth]{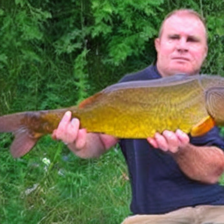} & \includegraphics[width=0.07\textwidth]{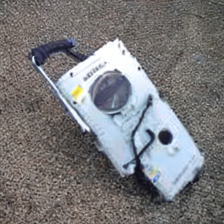} & \includegraphics[width=0.07\textwidth]{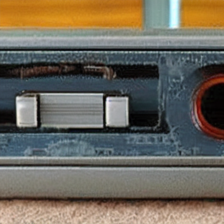} & \includegraphics[width=0.07\textwidth]{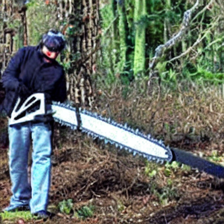} & \includegraphics[width=0.07\textwidth]{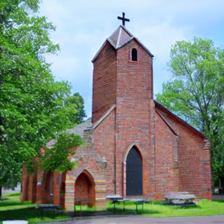} & \includegraphics[width=0.07\textwidth]{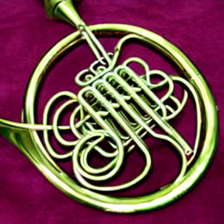} & \includegraphics[width=0.07\textwidth]{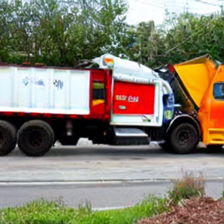} & \includegraphics[width=0.07\textwidth]{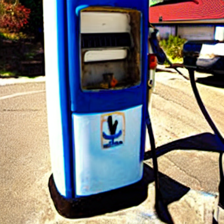} & \includegraphics[width=0.07\textwidth]{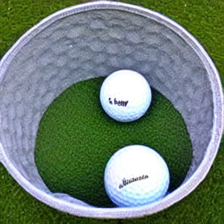} & \includegraphics[width=0.07\textwidth]{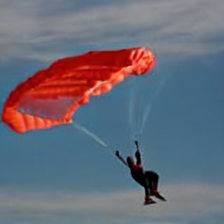} \\
Cassette player  & \includegraphics[width=0.07\textwidth]{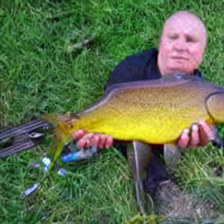} & \includegraphics[width=0.07\textwidth]{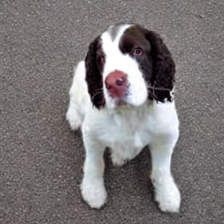} & \includegraphics[width=0.07\textwidth]{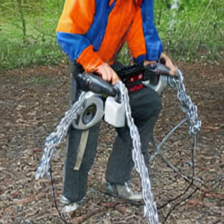} & \includegraphics[width=0.07\textwidth]{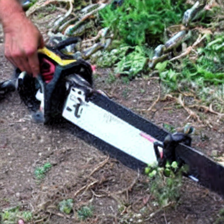} & \includegraphics[width=0.07\textwidth]{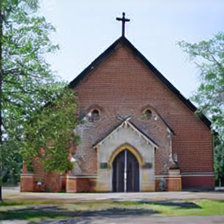} & \includegraphics[width=0.07\textwidth]{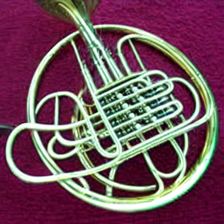} & \includegraphics[width=0.07\textwidth]{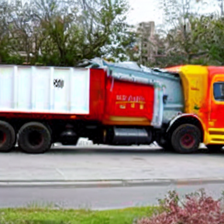} & \includegraphics[width=0.07\textwidth]{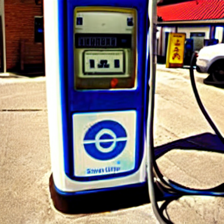} & \includegraphics[width=0.07\textwidth]{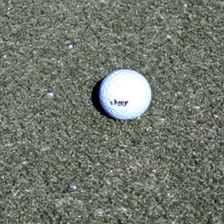} & \includegraphics[width=0.07\textwidth]{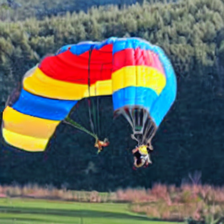} \\
Chain saw        & \includegraphics[width=0.07\textwidth]{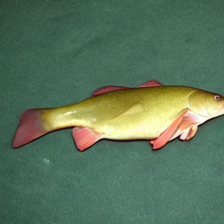} & \includegraphics[width=0.07\textwidth]{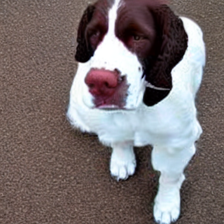} & \includegraphics[width=0.07\textwidth]{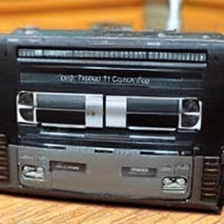} & \includegraphics[width=0.07\textwidth]{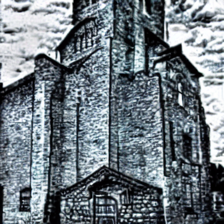} & \includegraphics[width=0.07\textwidth]{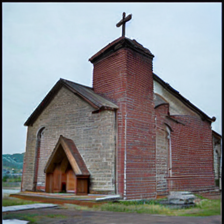} & \includegraphics[width=0.07\textwidth]{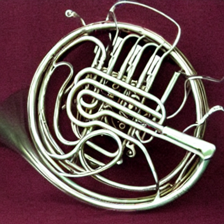} & \includegraphics[width=0.07\textwidth]{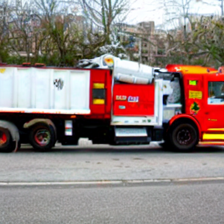} & \includegraphics[width=0.07\textwidth]{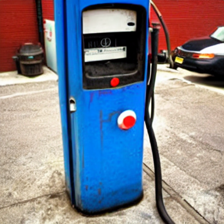} & \includegraphics[width=0.07\textwidth]{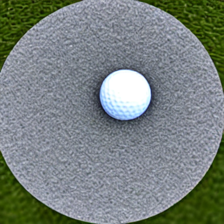} & \includegraphics[width=0.07\textwidth]{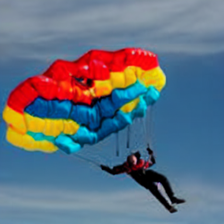} \\
Church           & \includegraphics[width=0.07\textwidth]{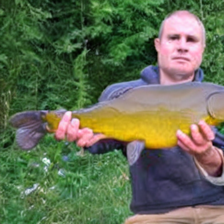} & \includegraphics[width=0.07\textwidth]{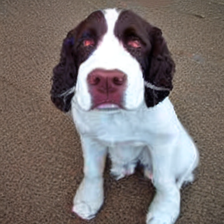} & \includegraphics[width=0.07\textwidth]{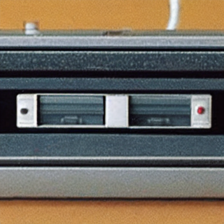} & \includegraphics[width=0.07\textwidth]{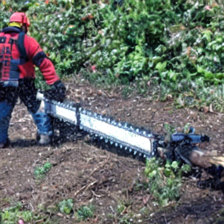} & \includegraphics[width=0.07\textwidth]{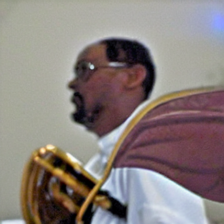} & \includegraphics[width=0.07\textwidth]{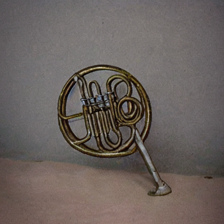} & \includegraphics[width=0.07\textwidth]{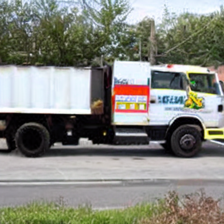} & \includegraphics[width=0.07\textwidth]{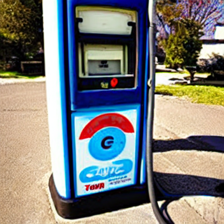} & \includegraphics[width=0.07\textwidth]{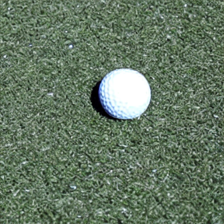} & \includegraphics[width=0.07\textwidth]{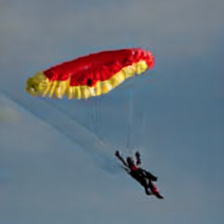} \\
French horn      & \includegraphics[width=0.07\textwidth]{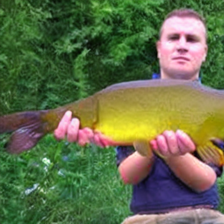} & \includegraphics[width=0.07\textwidth]{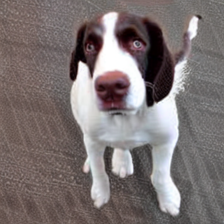} & \includegraphics[width=0.07\textwidth]{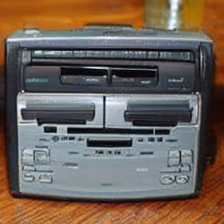} & \includegraphics[width=0.07\textwidth]{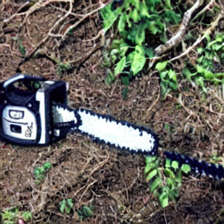} & \includegraphics[width=0.07\textwidth]{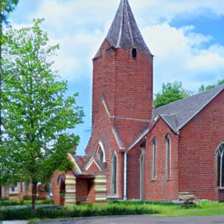} & \includegraphics[width=0.07\textwidth]{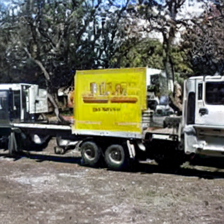} & \includegraphics[width=0.07\textwidth]{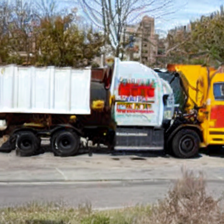} & \includegraphics[width=0.07\textwidth]{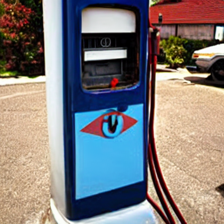} & \includegraphics[width=0.07\textwidth]{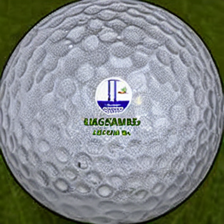} & \includegraphics[width=0.07\textwidth]{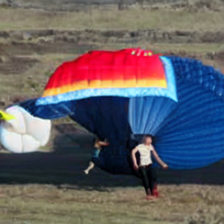} \\
Garbage truck    & \includegraphics[width=0.07\textwidth]{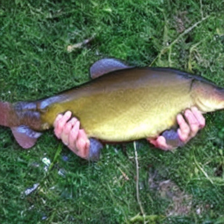} & \includegraphics[width=0.07\textwidth]{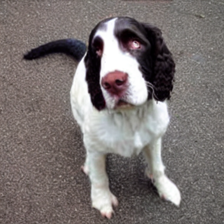} & \includegraphics[width=0.07\textwidth]{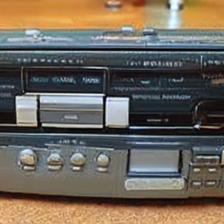} & \includegraphics[width=0.07\textwidth]{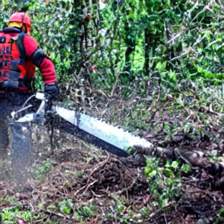} & \includegraphics[width=0.07\textwidth]{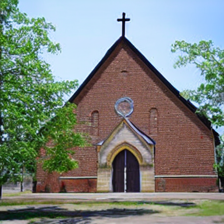} & \includegraphics[width=0.07\textwidth]{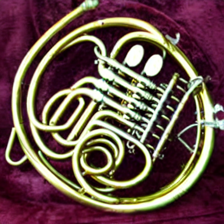} & \includegraphics[width=0.07\textwidth]{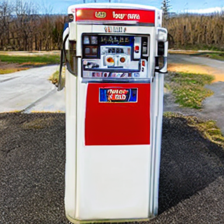} & \includegraphics[width=0.07\textwidth]{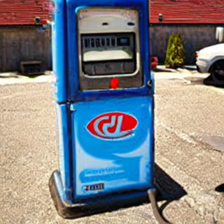} & \includegraphics[width=0.07\textwidth]{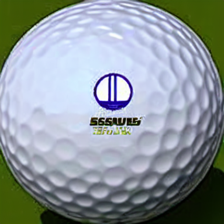} & \includegraphics[width=0.07\textwidth]{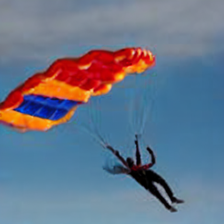} \\
Gas pump         & \includegraphics[width=0.07\textwidth]{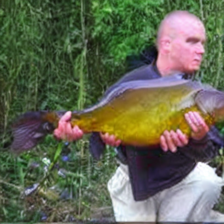} & \includegraphics[width=0.07\textwidth]{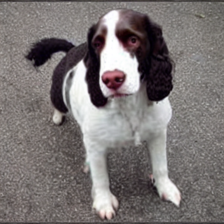} & \includegraphics[width=0.07\textwidth]{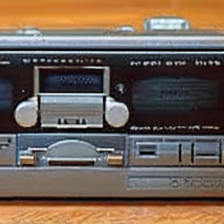} & \includegraphics[width=0.07\textwidth]{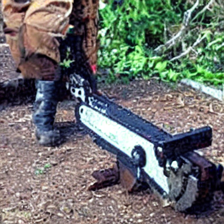} & \includegraphics[width=0.07\textwidth]{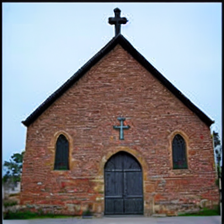} & \includegraphics[width=0.07\textwidth]{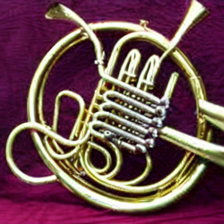} & \includegraphics[width=0.07\textwidth]{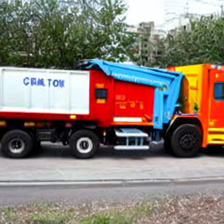} & \includegraphics[width=0.07\textwidth]{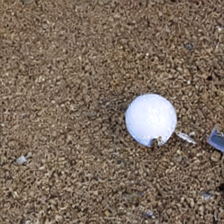} & \includegraphics[width=0.07\textwidth]{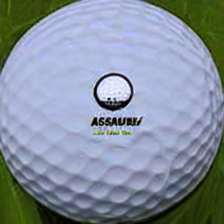} & \includegraphics[width=0.07\textwidth]{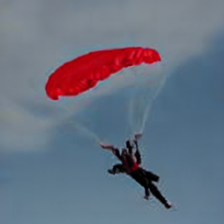} \\
Golf ball        & \includegraphics[width=0.07\textwidth]{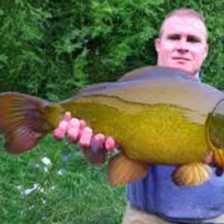} & \includegraphics[width=0.07\textwidth]{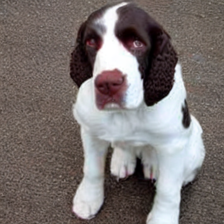} & \includegraphics[width=0.07\textwidth]{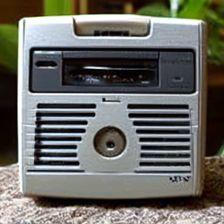} & \includegraphics[width=0.07\textwidth]{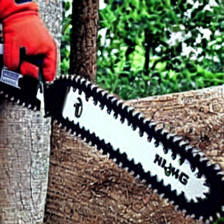} & \includegraphics[width=0.07\textwidth]{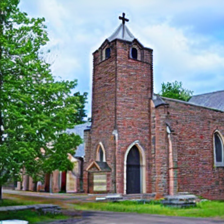} & \includegraphics[width=0.07\textwidth]{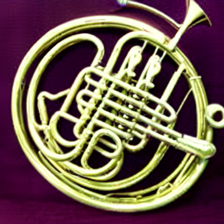} & \includegraphics[width=0.07\textwidth]{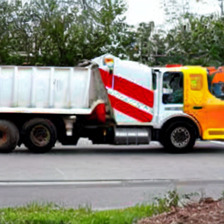} & \includegraphics[width=0.07\textwidth]{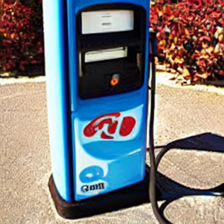} & \includegraphics[width=0.07\textwidth]{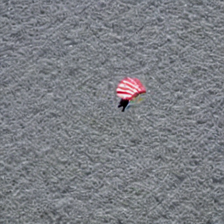} & \includegraphics[width=0.07\textwidth]{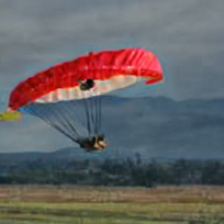} \\
Parachute        & \includegraphics[width=0.07\textwidth]{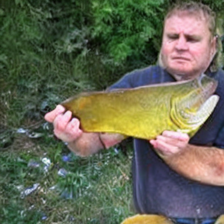} & \includegraphics[width=0.07\textwidth]{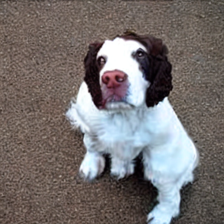} & \includegraphics[width=0.07\textwidth]{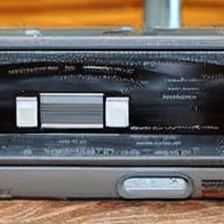} & \includegraphics[width=0.07\textwidth]{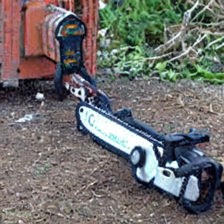} & \includegraphics[width=0.07\textwidth]{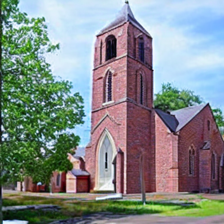} & \includegraphics[width=0.07\textwidth]{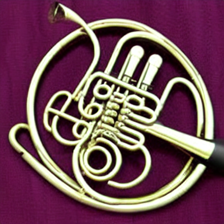} & \includegraphics[width=0.07\textwidth]{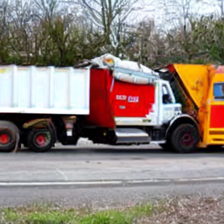} & \includegraphics[width=0.07\textwidth]{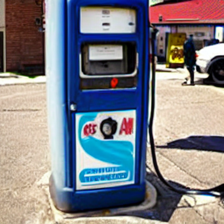} & \includegraphics[width=0.07\textwidth]{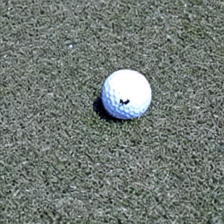} & \includegraphics[width=0.07\textwidth]{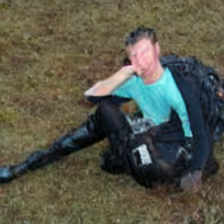} \\
\bottomrule
\end{tabular}
}
\end{figure*}

\begin{figure*}[t]
\centering
\renewcommand{\arraystretch}{1.1}
\setlength{\tabcolsep}{2pt}
\caption{Examples of generated images using \texttt{SSU}. From the rows below, diagonal images represent the forgetting class, while non-diagonal images represent the remaining class.}
\label{fig:imagenette_grid_2}
\resizebox{\textwidth}{!}{\begin{tabular}{c|cccccccccc}
\toprule
\textbf{Unlearned} & \multicolumn{10}{c}{\textbf{Prompt class}} \\
\textbf{class} & Tench & springer & Cassette & Saw & Church & French horn & truck & Gas pump & Golf ball & Parachute \\
\midrule
Tench            & \includegraphics[width=0.07\textwidth]{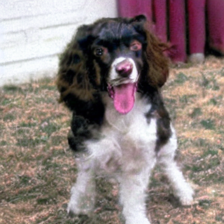} & \includegraphics[width=0.07\textwidth]{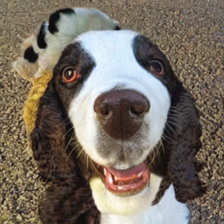} & \includegraphics[width=0.07\textwidth]{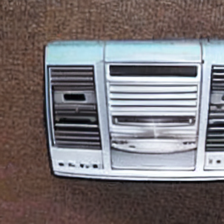} & \includegraphics[width=0.07\textwidth]{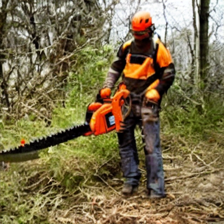} & \includegraphics[width=0.07\textwidth]{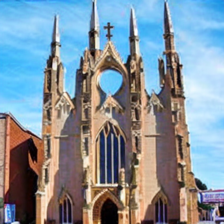} & \includegraphics[width=0.07\textwidth]{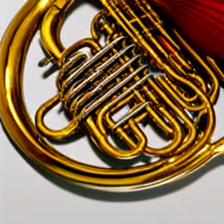} & \includegraphics[width=0.07\textwidth]{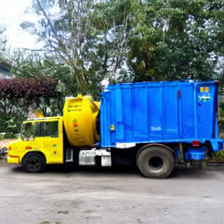} & \includegraphics[width=0.07\textwidth]{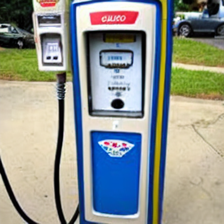} & \includegraphics[width=0.07\textwidth]{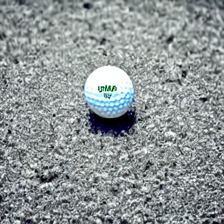} & \includegraphics[width=0.07\textwidth]{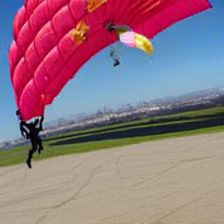} \\
English springer & \includegraphics[width=0.07\textwidth]{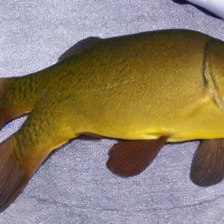} & \includegraphics[width=0.07\textwidth]{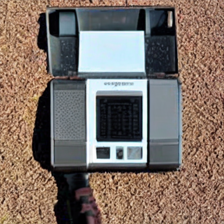} & \includegraphics[width=0.07\textwidth]{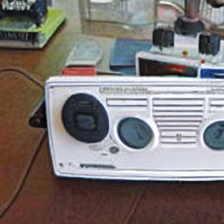} & \includegraphics[width=0.07\textwidth]{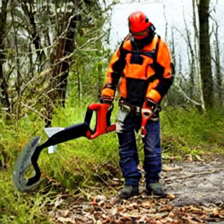} & \includegraphics[width=0.07\textwidth]{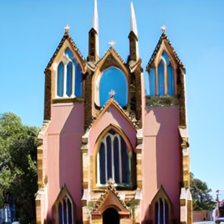} & \includegraphics[width=0.07\textwidth]{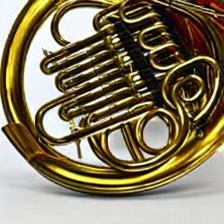} & \includegraphics[width=0.07\textwidth]{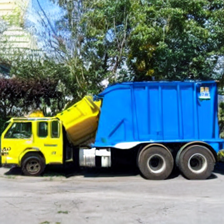} & \includegraphics[width=0.07\textwidth]{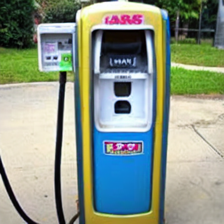} & \includegraphics[width=0.07\textwidth]{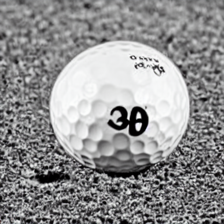} & \includegraphics[width=0.07\textwidth]{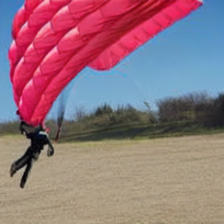} \\
Cassette player  & \includegraphics[width=0.07\textwidth]{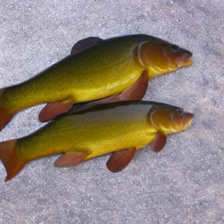} & \includegraphics[width=0.07\textwidth]{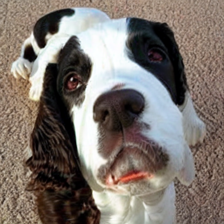} & \includegraphics[width=0.07\textwidth]{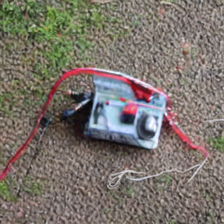} & \includegraphics[width=0.07\textwidth]{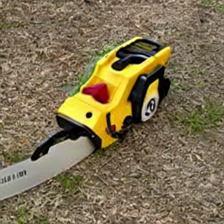} & \includegraphics[width=0.07\textwidth]{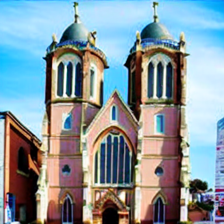} & \includegraphics[width=0.07\textwidth]{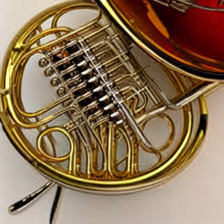} & \includegraphics[width=0.07\textwidth]{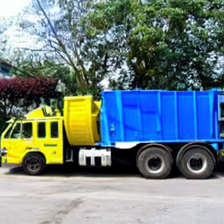} & \includegraphics[width=0.07\textwidth]{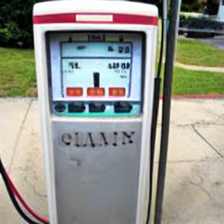} & \includegraphics[width=0.07\textwidth]{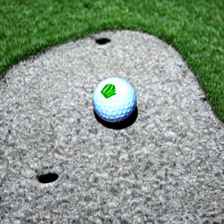} & \includegraphics[width=0.07\textwidth]{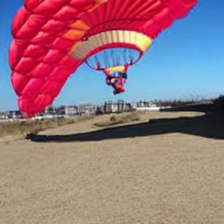} \\
Chain saw        & \includegraphics[width=0.07\textwidth]{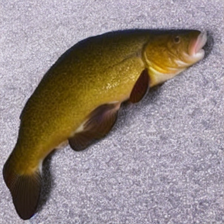} & \includegraphics[width=0.07\textwidth]{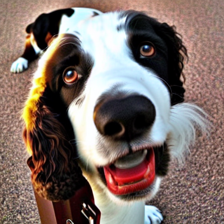} & \includegraphics[width=0.07\textwidth]{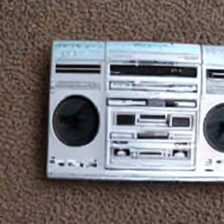} & \includegraphics[width=0.07\textwidth]{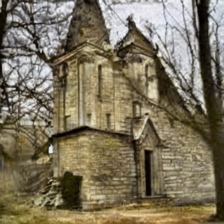} & \includegraphics[width=0.07\textwidth]{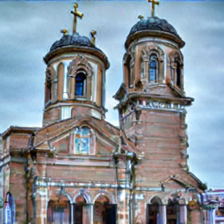} & \includegraphics[width=0.07\textwidth]{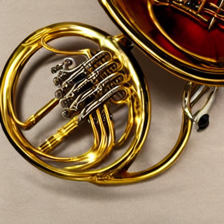} & \includegraphics[width=0.07\textwidth]{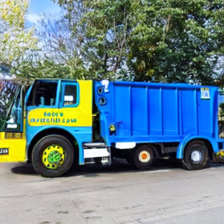} & \includegraphics[width=0.07\textwidth]{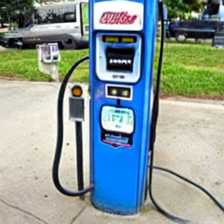} & \includegraphics[width=0.07\textwidth]{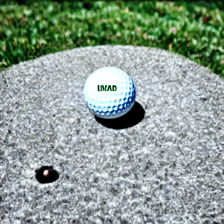} & \includegraphics[width=0.07\textwidth]{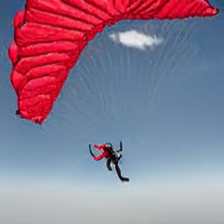} \\
Church           & \includegraphics[width=0.07\textwidth]{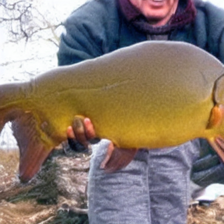} & \includegraphics[width=0.07\textwidth]{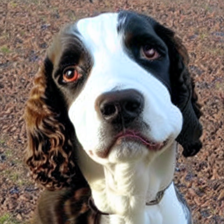} & \includegraphics[width=0.07\textwidth]{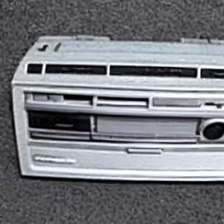} & \includegraphics[width=0.07\textwidth]{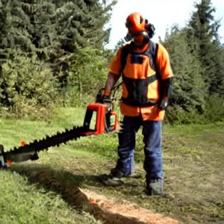} & \includegraphics[width=0.07\textwidth]{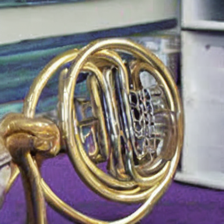} & \includegraphics[width=0.07\textwidth]{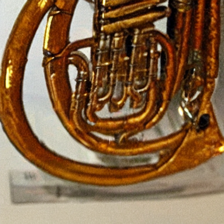} & \includegraphics[width=0.07\textwidth]{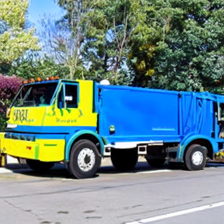} & \includegraphics[width=0.07\textwidth]{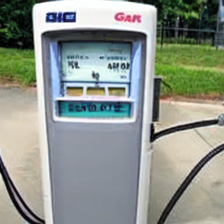} & \includegraphics[width=0.07\textwidth]{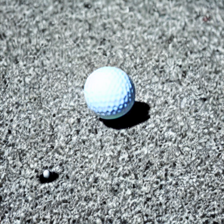} & \includegraphics[width=0.07\textwidth]{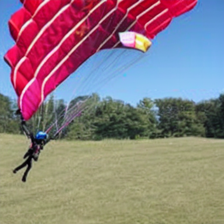} \\
French horn      & \includegraphics[width=0.07\textwidth]{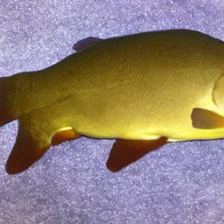} & \includegraphics[width=0.07\textwidth]{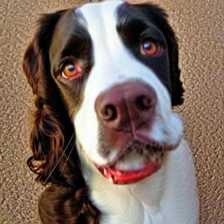} & \includegraphics[width=0.07\textwidth]{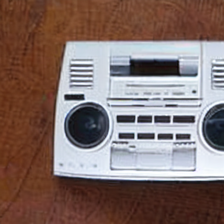} & \includegraphics[width=0.07\textwidth]{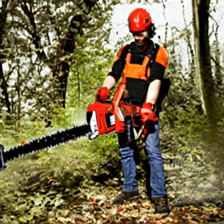} & \includegraphics[width=0.07\textwidth]{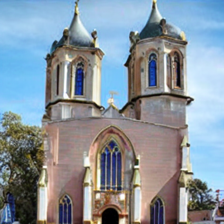} & \includegraphics[width=0.07\textwidth]{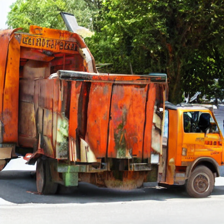} & \includegraphics[width=0.07\textwidth]{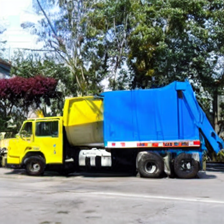} & \includegraphics[width=0.07\textwidth]{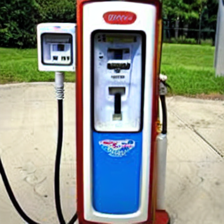} & \includegraphics[width=0.07\textwidth]{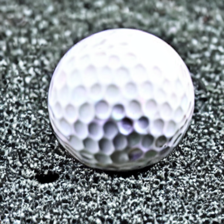} & \includegraphics[width=0.07\textwidth]{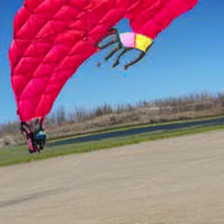} \\
Garbage truck    & \includegraphics[width=0.07\textwidth]{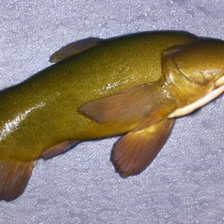} & \includegraphics[width=0.07\textwidth]{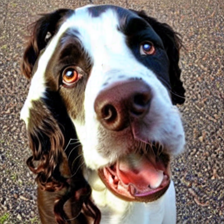} & \includegraphics[width=0.07\textwidth]{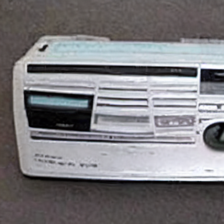} & \includegraphics[width=0.07\textwidth]{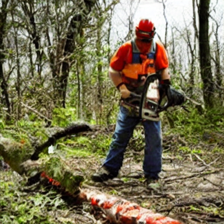} & \includegraphics[width=0.07\textwidth]{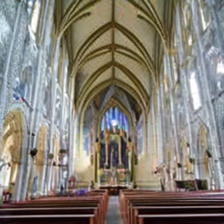} & \includegraphics[width=0.07\textwidth]{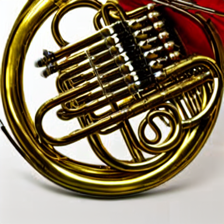} & \includegraphics[width=0.07\textwidth]{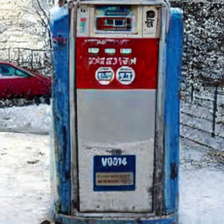} & \includegraphics[width=0.07\textwidth]{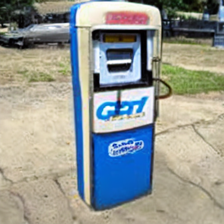} & \includegraphics[width=0.07\textwidth]{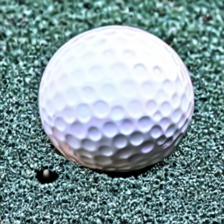} & \includegraphics[width=0.07\textwidth]{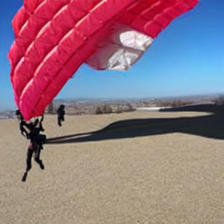} \\
Gas pump         & \includegraphics[width=0.07\textwidth]{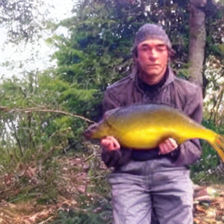} & \includegraphics[width=0.07\textwidth]{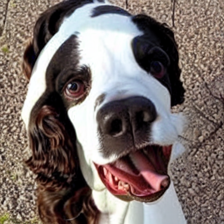} & \includegraphics[width=0.07\textwidth]{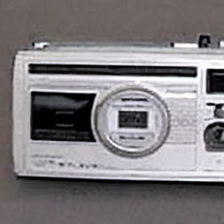} & \includegraphics[width=0.07\textwidth]{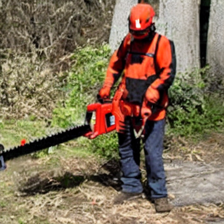} & \includegraphics[width=0.07\textwidth]{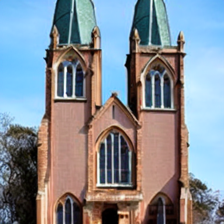} & \includegraphics[width=0.07\textwidth]{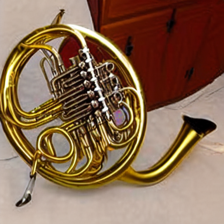} & \includegraphics[width=0.07\textwidth]{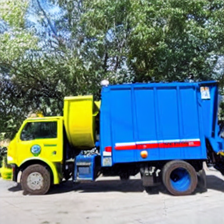} & \includegraphics[width=0.07\textwidth]{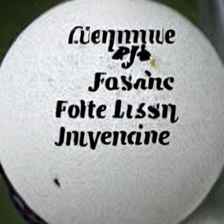} & \includegraphics[width=0.07\textwidth]{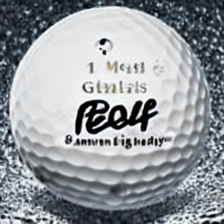} & \includegraphics[width=0.07\textwidth]{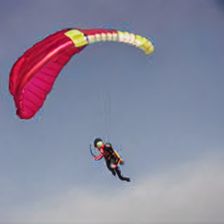} \\
Golf ball        & \includegraphics[width=0.07\textwidth]{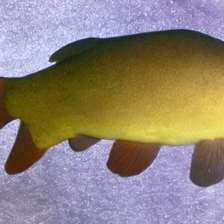} & \includegraphics[width=0.07\textwidth]{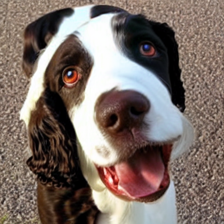} & \includegraphics[width=0.07\textwidth]{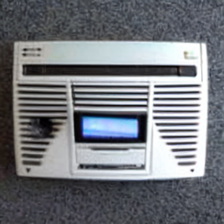} & \includegraphics[width=0.07\textwidth]{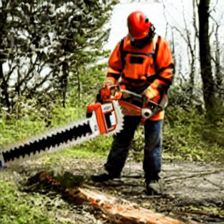} & \includegraphics[width=0.07\textwidth]{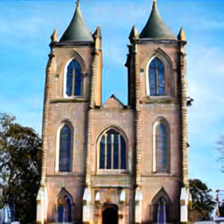} & \includegraphics[width=0.07\textwidth]{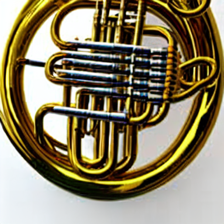} & \includegraphics[width=0.07\textwidth]{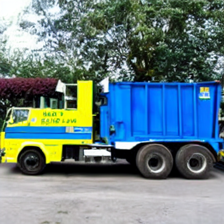} & \includegraphics[width=0.07\textwidth]{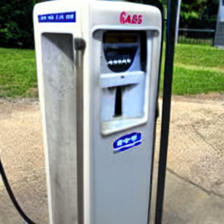} & \includegraphics[width=0.07\textwidth]{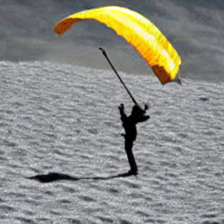} & \includegraphics[width=0.07\textwidth]{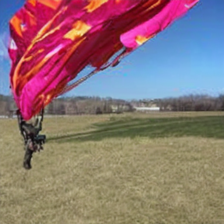} \\
Parachute        & \includegraphics[width=0.07\textwidth]{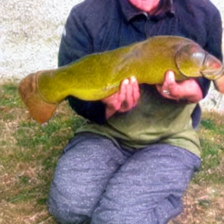} & \includegraphics[width=0.07\textwidth]{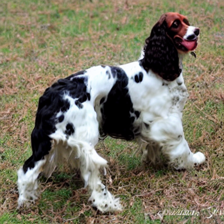} & \includegraphics[width=0.07\textwidth]{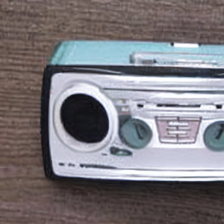} & \includegraphics[width=0.07\textwidth]{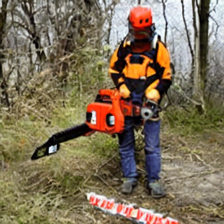} & \includegraphics[width=0.07\textwidth]{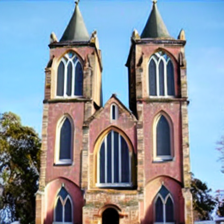} & \includegraphics[width=0.07\textwidth]{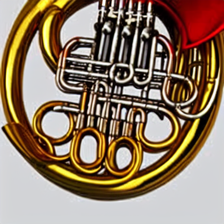} & \includegraphics[width=0.07\textwidth]{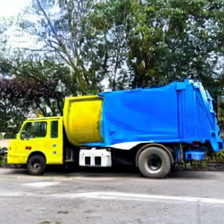} & \includegraphics[width=0.07\textwidth]{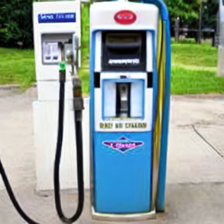} & \includegraphics[width=0.07\textwidth]{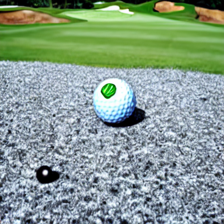} & \includegraphics[width=0.07\textwidth]{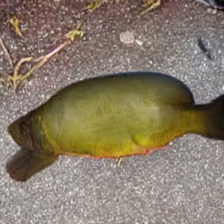} \\
\bottomrule
\end{tabular}
}
\end{figure*}

\begin{figure*}[t]
\centering
\renewcommand{\arraystretch}{1.1}
\setlength{\tabcolsep}{2pt}
\caption{Examples of generated images using \texttt{SSU}. From the rows below, diagonal images represent the forgetting class, while non-diagonal images represent the remaining class.}
\label{fig:imagenette_grid_3}
\resizebox{\textwidth}{!}{\begin{tabular}{c|cccccccccc}
\toprule
\textbf{Unlearned} & \multicolumn{10}{c}{\textbf{Prompt class}} \\
\textbf{class} & Tench & springer & Cassette & Saw & Church & French horn & truck & Gas pump & Golf ball & Parachute \\
\midrule
Tench            & \includegraphics[width=0.07\textwidth]{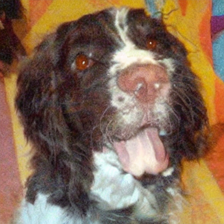} & \includegraphics[width=0.07\textwidth]{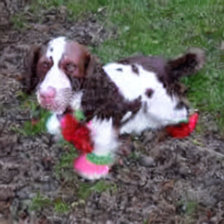} & \includegraphics[width=0.07\textwidth]{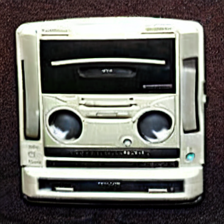} & \includegraphics[width=0.07\textwidth]{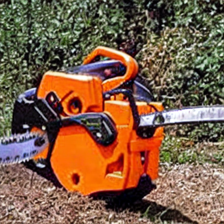} & \includegraphics[width=0.07\textwidth]{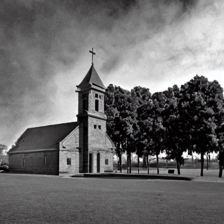} & \includegraphics[width=0.07\textwidth]{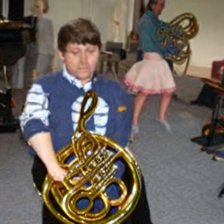} & \includegraphics[width=0.07\textwidth]{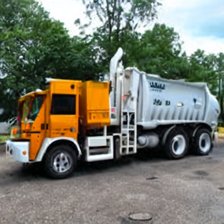} & \includegraphics[width=0.07\textwidth]{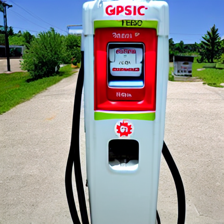} & \includegraphics[width=0.07\textwidth]{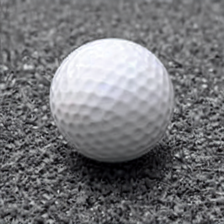} & \includegraphics[width=0.07\textwidth]{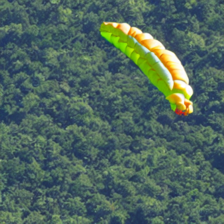} \\
English springer & \includegraphics[width=0.07\textwidth]{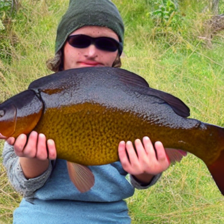} & \includegraphics[width=0.07\textwidth]{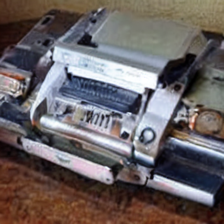} & \includegraphics[width=0.07\textwidth]{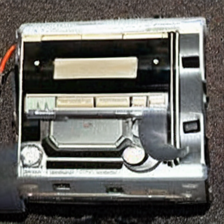} & \includegraphics[width=0.07\textwidth]{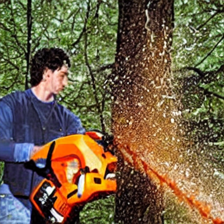} & \includegraphics[width=0.07\textwidth]{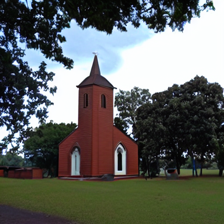} & \includegraphics[width=0.07\textwidth]{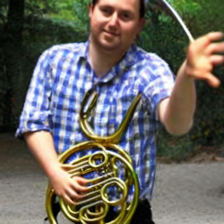} & \includegraphics[width=0.07\textwidth]{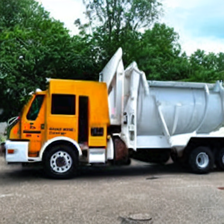} & \includegraphics[width=0.07\textwidth]{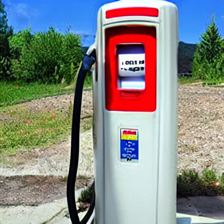} & \includegraphics[width=0.07\textwidth]{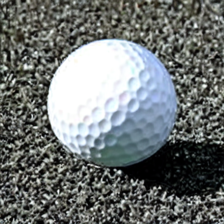} & \includegraphics[width=0.07\textwidth]{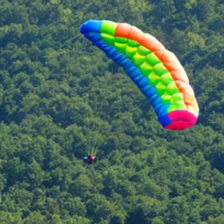} \\
Cassette player  & \includegraphics[width=0.07\textwidth]{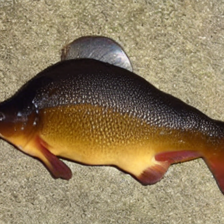} & \includegraphics[width=0.07\textwidth]{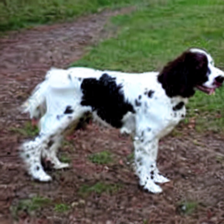} & \includegraphics[width=0.07\textwidth]{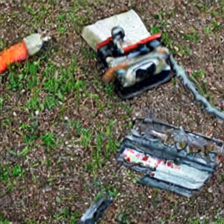} & \includegraphics[width=0.07\textwidth]{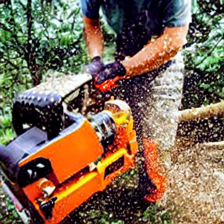} & \includegraphics[width=0.07\textwidth]{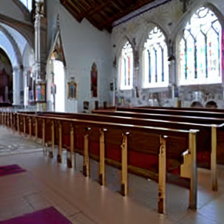} & \includegraphics[width=0.07\textwidth]{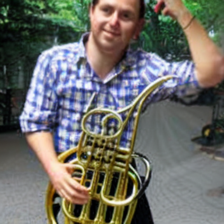} & \includegraphics[width=0.07\textwidth]{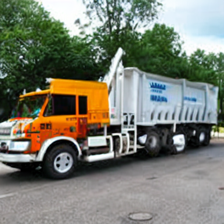} & \includegraphics[width=0.07\textwidth]{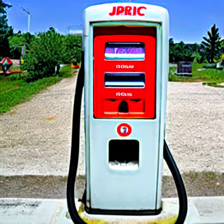} & \includegraphics[width=0.07\textwidth]{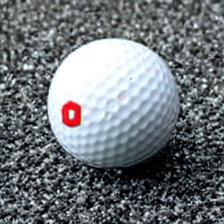} & \includegraphics[width=0.07\textwidth]{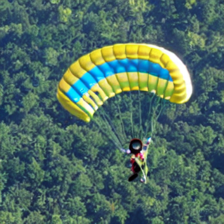} \\
Chain saw        & \includegraphics[width=0.07\textwidth]{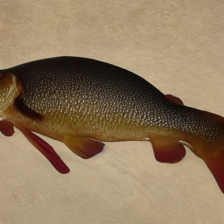} & \includegraphics[width=0.07\textwidth]{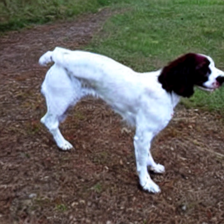} & \includegraphics[width=0.07\textwidth]{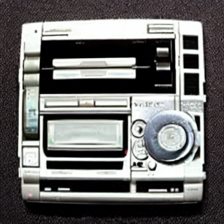} & \includegraphics[width=0.07\textwidth]{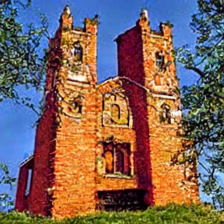} & \includegraphics[width=0.07\textwidth]{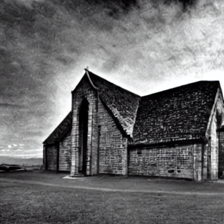} & \includegraphics[width=0.07\textwidth]{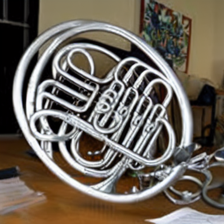} & \includegraphics[width=0.07\textwidth]{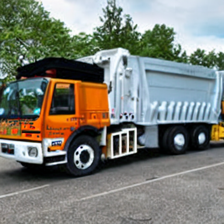} & \includegraphics[width=0.07\textwidth]{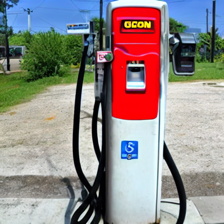} & \includegraphics[width=0.07\textwidth]{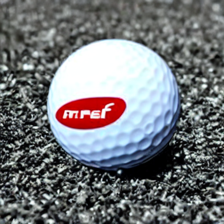} & \includegraphics[width=0.07\textwidth]{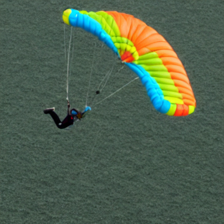} \\
Church           & \includegraphics[width=0.07\textwidth]{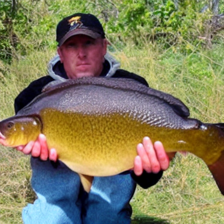} & \includegraphics[width=0.07\textwidth]{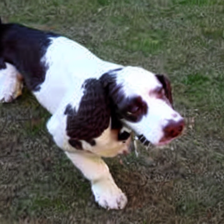} & \includegraphics[width=0.07\textwidth]{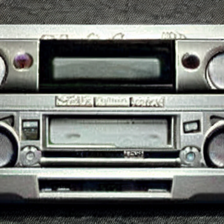} & \includegraphics[width=0.07\textwidth]{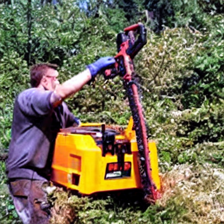} & \includegraphics[width=0.07\textwidth]{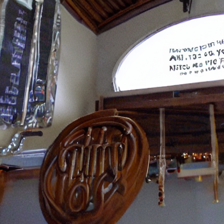} & \includegraphics[width=0.07\textwidth]{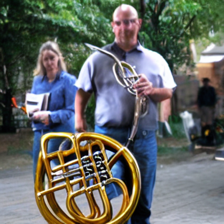} & \includegraphics[width=0.07\textwidth]{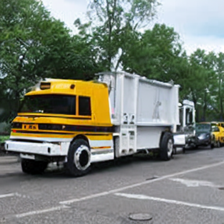} & \includegraphics[width=0.07\textwidth]{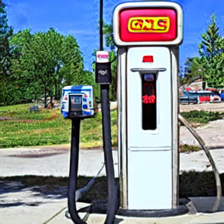} & \includegraphics[width=0.07\textwidth]{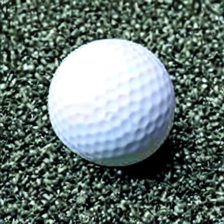} & \includegraphics[width=0.07\textwidth]{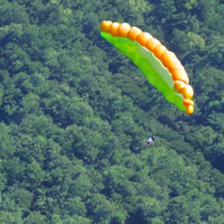} \\
French horn      & \includegraphics[width=0.07\textwidth]{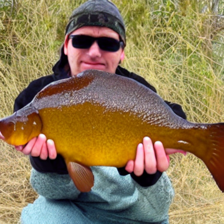} & \includegraphics[width=0.07\textwidth]{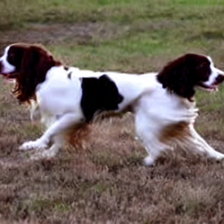} & \includegraphics[width=0.07\textwidth]{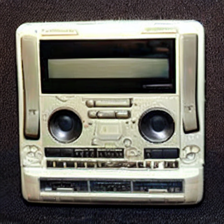} & \includegraphics[width=0.07\textwidth]{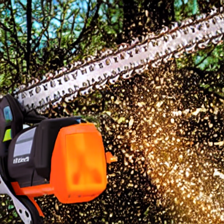} & \includegraphics[width=0.07\textwidth]{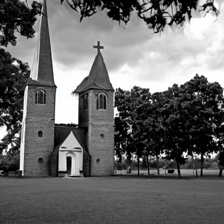} & \includegraphics[width=0.07\textwidth]{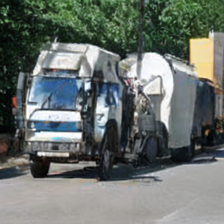} & \includegraphics[width=0.07\textwidth]{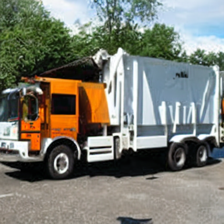} & \includegraphics[width=0.07\textwidth]{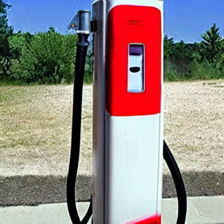} & \includegraphics[width=0.07\textwidth]{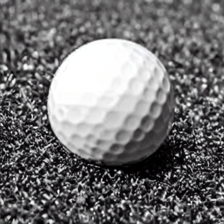} & \includegraphics[width=0.07\textwidth]{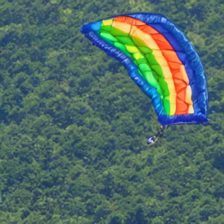} \\
Garbage truck    & \includegraphics[width=0.07\textwidth]{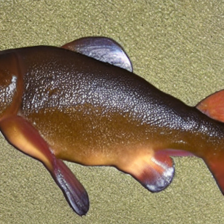} & \includegraphics[width=0.07\textwidth]{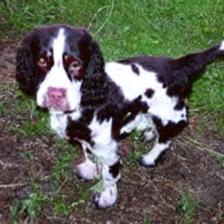} & \includegraphics[width=0.07\textwidth]{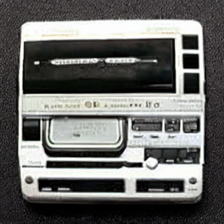} & \includegraphics[width=0.07\textwidth]{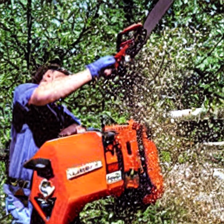} & \includegraphics[width=0.07\textwidth]{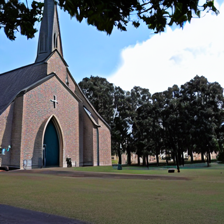} & \includegraphics[width=0.07\textwidth]{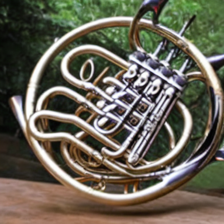} & \includegraphics[width=0.07\textwidth]{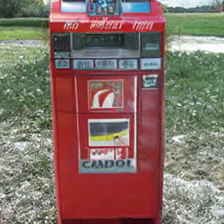} & \includegraphics[width=0.07\textwidth]{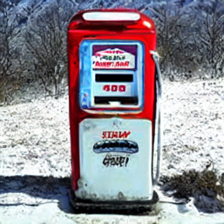} & \includegraphics[width=0.07\textwidth]{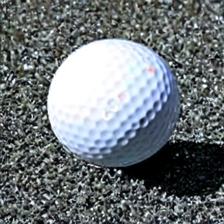} & \includegraphics[width=0.07\textwidth]{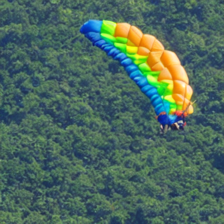} \\
Gas pump         & \includegraphics[width=0.07\textwidth]{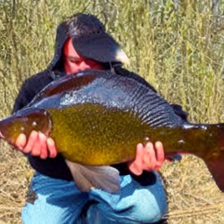} & \includegraphics[width=0.07\textwidth]{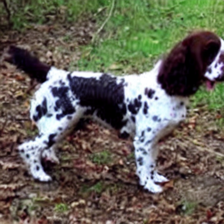} & \includegraphics[width=0.07\textwidth]{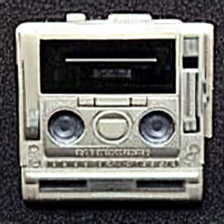} & \includegraphics[width=0.07\textwidth]{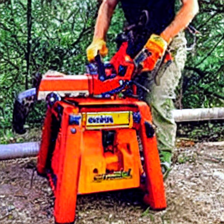} & \includegraphics[width=0.07\textwidth]{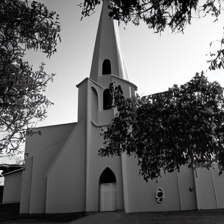} & \includegraphics[width=0.07\textwidth]{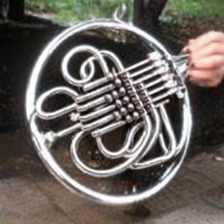} & \includegraphics[width=0.07\textwidth]{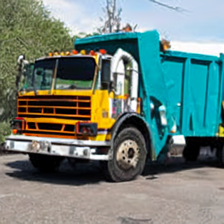} & \includegraphics[width=0.07\textwidth]{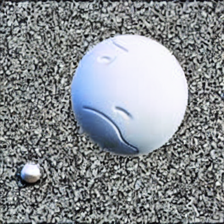} & \includegraphics[width=0.07\textwidth]{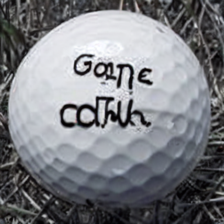} & \includegraphics[width=0.07\textwidth]{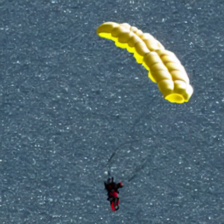} \\
Golf ball        & \includegraphics[width=0.07\textwidth]{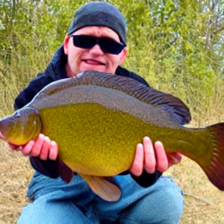} & \includegraphics[width=0.07\textwidth]{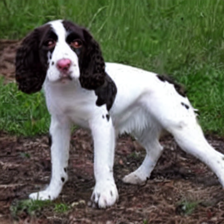} & \includegraphics[width=0.07\textwidth]{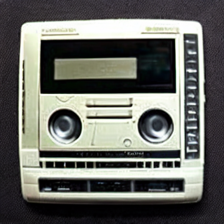} & \includegraphics[width=0.07\textwidth]{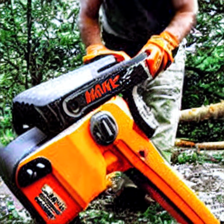} & \includegraphics[width=0.07\textwidth]{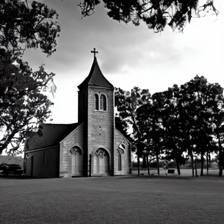} & \includegraphics[width=0.07\textwidth]{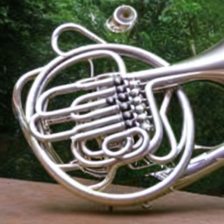} & \includegraphics[width=0.07\textwidth]{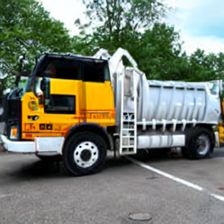} & \includegraphics[width=0.07\textwidth]{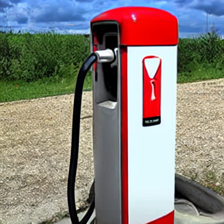} & \includegraphics[width=0.07\textwidth]{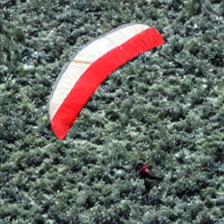} & \includegraphics[width=0.07\textwidth]{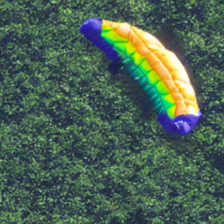} \\
Parachute        & \includegraphics[width=0.07\textwidth]{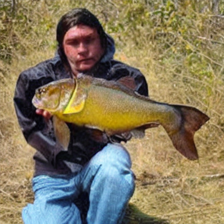} & \includegraphics[width=0.07\textwidth]{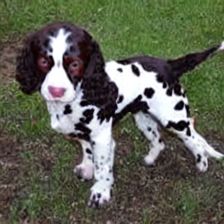} & \includegraphics[width=0.07\textwidth]{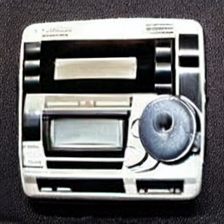} & \includegraphics[width=0.07\textwidth]{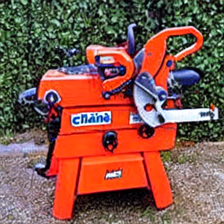} & \includegraphics[width=0.07\textwidth]{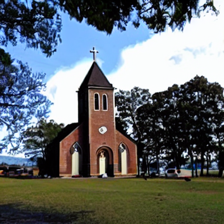} & \includegraphics[width=0.07\textwidth]{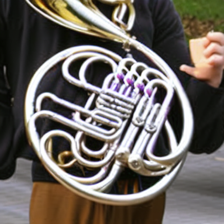} & \includegraphics[width=0.07\textwidth]{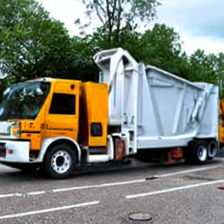} & \includegraphics[width=0.07\textwidth]{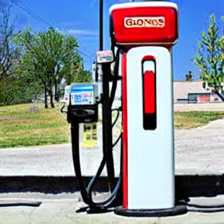} & \includegraphics[width=0.07\textwidth]{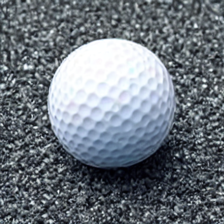} & \includegraphics[width=0.07\textwidth]{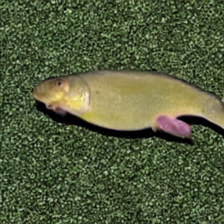} \\
\bottomrule
\end{tabular}
}
\end{figure*}

\figref{fig:imagenette_grid_1}, \figref{fig:imagenette_grid_2}, and \figref{fig:imagenette_grid_3}~present class-wise unlearning results on the Imagenette dataset using the SSU method, evaluated under different random seeds. Each figure is organized as a matrix, where rows indicate the ``Unlearned class'' and columns indicate the ``Prompt class,'' clearly separating the intended unlearning target from the generated outputs. Diagonal images correspond to the class being unlearned, highlighting the effectiveness of SSU in removing specific concepts. Off-diagonal images show generations for other classes, demonstrating the model's ability to generalize and distinguish among the remaining categories.

\end{document}